\documentclass[letterpaper, 11pt, journal, onecolumn]{ieeeconf}

\usepackage{amsmath,amssymb,graphicx,paralist,subfigure,pdfpages,cite,algorithm,algpseudocode,paralist}
\newtheorem{lem}{Lemma} 
\newtheorem{theorem}{Theorem}

\newtheorem{assump}{Assumption}

\def\ln{{\rm ln}}

\def\mc{\mathcal}
\def\mb{\mathbf}
\def\mbb{\mathbb}
\def\ra{\rightarrow}

\IEEEoverridecommandlockouts

\renewcommand{\baselinestretch}{1.5}

\def\mbb{\mathbb}
\def\mb{\mathbf}
\def\mc{\mathcal}

\def\ol{\overline}
\def\ul{\underline}

\def\bds{\boldsymbol}
\newcommand{\mn}[1]{{\left\vert\kern-0.25ex\left\vert\kern-0.25ex\left\vert\kern0.3ex #1 
    \kern0.3ex\right\vert\kern-0.25ex\right\vert\kern-0.25ex\right\vert}}

\begin{document}
\title{\textbf{Distributed stochastic optimization with gradient tracking\\ over strongly-connected networks}} 
\author{Ran Xin,  Anit Kumar Sahu, Usman A. Khan and Soummya Kar
\thanks{RX and UAK are with the Department of Electrical and Computer Engineering, Tufts University, Medford, MA, 02155; {\texttt{ran.xin@tufts.edu, khan@ece.tufts.edu}}. AKS is with the Bosch Center for Artificial Intelligence, Pittsburgh, PA; {\texttt{anit.sahu@gmail.com}}. SK is with the Department of Electrical and Computer Engineering, Carnegie Mellon University, Pittsburgh, PA, 15213; {\texttt{soummyak@andrew.cmu.edu}}. The work of UAK and RX has been partially supported by an NSF Career Award \# CCF-1350264. The work of SK has been partially supported by NSF under grant \# CCF-1513936.}}

\maketitle
\thispagestyle{empty}

\begin{abstract}
In this paper, we study distributed stochastic optimization to minimize a sum of smooth and strongly-convex local cost functions over a network of agents, communicating over a strongly-connected graph. Assuming that each agent has access to a stochastic first-order oracle~($\mathcal{SFO}$), we propose a novel distributed method, called~$\mathcal{S}$-$\mathcal{AB}$, where each agent uses an auxiliary variable to asymptotically track the gradient of the global cost in expectation. The~$\mathcal{S}$-$\mathcal{AB}$ algorithm employs row- and column-stochastic weights simultaneously to ensure both consensus and optimality. Since doubly-stochastic weights are not used,~$\mathcal{S}$-$\mathcal{AB}$ is applicable to arbitrary strongly-connected graphs. We show that under a sufficiently small constant step-size,~$\mathcal{S}$-$\mathcal{AB}$ converges linearly (in expected mean-square sense) to a neighborhood of the global minimizer. We present numerical simulations based on real-world data sets to illustrate the theoretical results.  

\end{abstract}

\section{Introduction}
\noindent In the era of data deluge, where it is particularly difficult to store and process all data on a single device/node/processor, distributed schemes are becoming attractive for inference, learning, and optimization. Distributed optimization over multi-agent systems, thus, has been of significant interest in many areas including but not limited to machine learning~\cite{mcmahan2017communication,RajaBajwa.ITSP16}, big-data analytics~\cite{daneshmand2015hybrid,DNC}, and distributed control~\cite{BulloBook,8123915}. However, the underlying algorithms must be designed to address practical limitations and realistic scenarios. For instance, with the computation and data collection/storage being pushed to the edge devices, e.g., in Internet of Things~(IoT), the data available for distributed optimization is often inexact. Moreover, the ad hoc nature of setups outside of data centers requires the algorithms to be amenable to communication protocols that are not necessarily bidirectional. The focus of this paper is to study and characterize distributed optimization schemes where the inter-agent communication is restricted to directed graphs and the information/data is inexact.

In particular, we study distributed stochastic optimization over directed graphs and propose the~$\mc{S}$-$\mc{AB}$ algorithm to minimize a sum of local cost functions. The~$\mc{S}$-$\mc{AB}$ algorithm assumes access to a stochastic first-order oracle~($\mc{SFO}$), i.e., when an agent queries the~$\mc{SFO}$, it gets an unbiased estimate of the gradient of its local cost function. In the proposed approach, each agent makes a weighted average of its own and its neighbors’ solution estimates, and simultaneously incorporates its local gradient estimate of the global cost function. The exchange of solution estimates is performed over a row-stochastic weight matrix. In parallel, each agent maintains its own estimate of the gradient of the global cost function, by simultaneously incorporating a weighted average of its and its neighbors' gradient estimates and its local gradient tracking estimate. The exchange of gradient estimates of the global cost function is performed over a column-stochastic weight matrix. Since doubly-stoachstic weights are nowhere used,~$\mc{S}$-$\mc{AB}$ is an attractive solution that is applicable to arbitrary, strongly-connected graphs.

The main contributions of this paper are as follows: 
\begin{inparaenum}[(i)]
\item We show that, by choosing a sufficiently small constant~step-size,~$\alpha$, $\mc{S}$-$\mc{AB}$ converges linearly~to a neighborhood of the global minimizer. This convergence guarantee is achieved for continuously-differentiable, strongly-convex, local cost functions, where each agent is assumed to have access to a~$\mc{SFO}$ and the gradient noise has zero-mean and bounded variance. 
\item We provide explicit expressions of the appropriate norms under which the row- and column-stochastic weight matrices contract. With the help of these norms, we develop sharp and explicit convergence arguments. 
\end{inparaenum}

We now briefly review the literature concerning distributed and stochastic optimization. Early work on deterministic finite-sum problems includes~\cite{uc_Nedic,cc_lobel,cc_Duchi}, while work on stochastic problems can be found in~\cite{ram2010distributed,NedicStochasticPush}. Recently, gradient tracking has been proposed where the local gradient at each agent is replaced by the estimate of the global gradient~\cite{AugDGM,harness,diging,add-opt}. Methods for directed graphs that are  based on gradient tracking~\cite{opdirect_Tsianous,opdirect_Nedic,D-DGD,D-DPS,diging,add-opt,linear_row,FROST} rely on separate iterations for eigenvector estimation that may impede the convergence. This issue was recently resolved in~\cite{AB,pushpull_Pu}, see also~\cite{ABm,Gossippushpull,TVAB,NCAB} for the follow-up work, where eigenvector estimation was removed with the help of a unique approach that uses both row- and column-stochastic weights. Ref.~\cite{AB} derives linear convergence of the finite-sum problem when the underlying functions are smooth and strongly-convex, however, since arbitrary norms are used in the analysis, the convergence bounds are not sharp. Recent related work on time-varying networks and other approaches can be found in~\cite{DistributedMirrorDescent,Kozat,jakovetic2018convergence,sahu2018distributed,jakovetic2018unification}, albeit, without gradient tracking. Of significant relevance is~\cite{pu2018distributed}, where a similar setup with gradient tracking is considered over undirected graphs. We note that~$\mc{S}$-$\mc{AB}$  generalizes~\cite{pu2018distributed} and the analysis in~\cite{pu2018distributed} relies on the weight matrix contraction in~$2$-norm that is not applicable here. 

We now describe the rest of the paper. Section~\ref{pf} describes the problem, assumptions, and some auxiliary results. We present the convergence analysis in Section~\ref{conv} and the main result in Section~\ref{main}. Finally, Section~\ref{sims} provides the numerical experiments and Section~\ref{conc} concludes the paper. 

\textbf{Basic Notation:} We use lowercase bold letters for vectors and uppercase italic letters for matrices. We use~$I_n$ for the~$n\times n$ identity matrix, and~$\mb{1}_n$ for the column of~$n$ ones. For an arbitrary vector,~$\mb{x}$, we denote its~$i$th element  by~$[\mb{x}]_i$ and its smallest element by~$\ul{\mb{x}}$ and its largest element by~$\ol{\mb{x}}$. Inequalities involving matrices and vectors are to be interpreted componentwise. For a matrix,~$X$, we denote~$\rho(X)$ as its spectral radius and~$X_\infty$ as its infinite power (if it exists), i.e.,~$X_\infty=\lim_{k\ra\infty}X^k$. For a primitive, row-stochastic matrix,~$A$, we denote its left and right eigenvectors corresponding to the eigenvalue of~$1$ by~$\bds{\pi}_r$ and~$\mb{1}_n$, respectively, such that~$\bds{\pi}_r^\top\mb{1}_n = 1$ and~$A_\infty=\mb{1}_n\bds{\pi}_r^\top$. Similarly, for a primitive, column-stochastic matrix,~$B$, we have~$B_\infty=\bds{\pi}_c\mb{1}_n^\top$. 

\section{Problem formulation and Auxiliary Results}\label{pf}
\noindent Consider~$n$ agents connected over a directed graph,~$\mc{G}=(\mc{V},\mc{E})$, where~$\mc{V}=\{1,\cdots,n\}$ is the set of agents, and~$\mc{E}$ is the collection of ordered pairs,~$(i,j),i,j\in\mc{V}$, such that agent~$j$ can send information to agent~$i$. We assume that~$(i,i)\in\mc{E},\forall i$. The agents solve the following  problem:
\begin{align}
\label{eq:opt_problem}
\mbox{P1}:
\quad\min_{\mb{x}\in\mathbb{R}^p}F(\mb{x})\triangleq\frac{1}{n}\sum_{i=1}^nf_i(\mb{x}),
\end{align}
\noindent where each~$f_i:\mbb{R}^p\rightarrow\mbb{R}$ is known only to agent~$i$. We now formalize the assumptions.
\begin{assump}\label{asp1}
	Each local objective,~$f_i$, is~$\mu$-strongly-convex, i.e.,~$\forall i\in\mc{V}$ and~$\forall\mb{x}, \mb{y}\in\mbb{R}^p$. Thus, we have
	\begin{equation*}
	f_i(\mb{y})\geq f_i(\mb{x})+\nabla f_i(\mb{x})^\top(\mb{y}-\mb{x})+\frac{\mu}{2}\|\mb{x}-\mb{y}\|_2^2.
	\end{equation*}
\end{assump} 
\noindent Under Assumption \ref{asp1}, the optimal solution for Problem~$\mbox{P1}$ exists and is unique, which we denote as~$\mathbf{x}^{\ast}$. 
\noindent\begin{assump}\label{asp2}
	Each local objective,~$f_i$, is~$l$-smooth, i.e., its gradient is Lipschitz-continuous:~$\forall i\in\mc{V}$ and~$\forall\mb{x}, \mb{y}\in\mbb{R}^p$, we have, for some~$l>0$,
	\begin{equation*}
	\qquad\|\mb{\nabla} f_i(\mb{x})-\mb{\nabla} f_i(\mb{y})\|_2\leq l\|\mb{x}-\mb{y}\|_2.
	\end{equation*}
\end{assump}

We make the following assumption on the agent communication graph, which guarantees the existence of a \textit{directed} path from each agent $i$ to each agent $j$.
\noindent\begin{assump}\label{asp3}
	The  graph,~$\mc{G}$, is strongly-connected.
\end{assump}

We consider distributed iterative algorithms to solve Problem~\mbox{P1}, where each agent is able to call a stochastic first-order oracle~($\mc{SFO}$). At iteration~$k$ and agent~$i$, given~$\mb{x}^i_k\in\mathbb{R}^p$ as the input,~$\mc{SFO}$ returns a stochastic gradient in the form of~$\mb{g}_i(\mb{x}_{k}^i,\xi_k^i)\in\mathbb{R}^p$, where~$\xi_k^i\in\mathbb{R}^m$ are random vectors,~$\forall k\geq 0, \forall i\in\mc{V}$. The stochastic gradients,~$\mb{g}_i(\mb{x}_{k}^i,\xi_k^i)$, satisfy the following standard assumptions:   
\begin{assump}
	\label{asp4}
	The set of random vectors~$\{\xi_k^i\}_{k\geq0,i\in\mc{V}}$ are independent of each other, and 
	\begin{enumerate}[(1)]
		\item $\mathbb{E}_{\xi_k^i}\left[\mb{g}_i(\mb{x}_{k}^i,\xi_k^i)|\mb{x}_k^i\right]=\nabla f_i(\mb{x}_k^i)$,
		\item 
		$\mathbb{E}_{\xi_k^i}\left[\left\|\mb{g}_i(\mb{x}_k^i,\xi_k^i)-\nabla f_i(\mb{x}_k^i)\right\|_2^2|\mb{x}_k^i\right]\leq\sigma^2$.
	\end{enumerate}
\end{assump}
Assumption~\ref{asp4} is satisfied in many scenarios, for example, when the gradient noise,~$\mb{g}_i(\mb{x}_k^i,\xi_i)-\nabla f_i(\mb{x}_k^i)$, is independent and identically distributed~(i.i.d.) with zero-mean and finite second moment, while being independent of~$\mb{x}_k^i$. However, Assumption~\ref{asp4} allows for general gradient noise processes dependent on agent $i$ and the current iterate~$\mb{x}_k^i$.
Finally, we denote by~$\mc{F}_k$ the~$\sigma$-algebra generated by the set of random vectors~$\{\xi_t^i\}_{0\leq t\leq k-1,i\in\mc{V}}$.

\subsection{The~$\mc{S}$-$\mc{AB}$ algorithm}
\label{subsec:algo}
We now describe the proposed algorithm,~$\mc{S}$-$\mc{AB}$, to solve Problem~\mbox{P1}. Each agent~$i$ maintains two state vectors,~$\mb{x}_k^i$ and~$\mb{y}_k^i$, both in~$\mathbb{R}^p$, where~$k$ is the number of iterations. The variable~$\mb{x}_k^i$ is the estimate of the global minimizer~$\mb{x}^*$, while~$\mb{y}_k^i$ is the global gradient estimator. The~$\mc{S}$-$\mc{AB}$ algorithm, initialized with arbitrary~$\mb{x}^i_0$'s and with~$\mb{y}^i_0=\mb{g}_i(\mb{x}_{0}^i,\xi_0^i)$,~$\forall i\in\mc{V}$, is given by the following:
\begin{subequations}\label{SAB}
	\begin{align}
	\mb{x}_{k+1}^i &= \sum_{j=1}^{n}a_{ij}\mb{x}_k^j -\alpha\mb{y}_k^i, \label{SAB1}\\
	\mb{y}_{k+1}^i &= \sum_{j=1}^{n}b_{ij}\mb{y}_k^j + \mb{g}_i(\mb{x}_{k+1}^i,\xi_{k+1}^i)-\mb{g}_i(\mb{x}_{k}^i,\xi_{k}^i),\label{SAB2}
	\end{align}
\end{subequations}
where the weight matrices~$A=\{a_{ij}\}$ and~$B=\{b_{ij}\}$ are row- and column-stochastic, respectively, and follow the graph topology, i.e.,~$a_{ij}>0$ and~$b_{ij}>0$, iff~$(i,j)\in\mc{E}$. We next write the algorithm in a compact vector form for the sake of analysis.
\begin{subequations}\label{SABv}
\begin{align}
\mb{x}_{k+1} &= \mc{A}\mb{x}_k -\alpha\mb{y}_k, \label{SABv1}\\
\mb{y}_{k+1} &= \mc{B}\mb{y}_k + \mb{g}(\mb{x}_{k+1},\bds{\xi}_{k+1})-\mb{g}(\mb{x}_{k},\bds{\xi}_k),\label{SABv2}
\end{align}
\end{subequations}
where we use the following notation:
	\begin{align*}
	    \mb{x}_k\triangleq
	\left[\begin{array}{cc}
	\mb{x}_{k}^1\\\vdots\\\mb{x}_{k}^n\end{array}\right],~\mb{y}_k\triangleq
	\left[\begin{array}{cc}\mb{y}_{k}^1\\\vdots\\\mb{y}_{k}^n\end{array}\right],~
	\mb{g}(\mb{x}_{k},\bds{\xi}_k)\triangleq
	\left[\begin{array}{cc}\mb{g}_1(\mb{x}_{k}^1,\xi_{k}^1)\\\vdots\\
		\mb{g}_n(\mb{x}_{k}^n,\xi_{k}^n)\end{array}\right],
	\end{align*}
and~$\mc{A} = A\otimes I_p,\mc{B} = B\otimes I_p.$ 

Note that when the variance,~$\sigma$, of the stochastic gradients is~0, we recover the~$\mc{AB}$ or the push-pull algorithm proposed in~\cite{AB,pushpull_Pu}. In the following, we assume~$p=1$ for the sake of simplicity. The analysis can be extended to the general case of~$p>1$ with the help of Kronecker products. 

\subsection{Auxiliary Results}\label{aux_res}
We now provide some auxiliary results to aid the convergence analysis of~$\mc{S}$-$\mc{AB}$. We first develop explicit norms regarding the contractions of the weight matrices,~$A$ and~$B$. Since both~$A$ and~$B$ are primitive and stochastic, we use their non-$\mb{1}_n$ Perron vectors,~$\bds{\pi}_r$ and~$\bds{\pi}_c$, respectively, to define two weighted inner products as follows:~$\forall \mb{x},\mb{y}\in\mathbb{R}^n$,
\begin{align*}
&\langle\mb{x},\mb{y}\rangle_{\bds{\pi}_r}\triangleq\mb{x}^\top\mbox{diag}(\bds{\pi}_r)\mb{y}, \\
&\langle\mb{x},\mb{y}\rangle_{\bds{\pi}_c}\triangleq\mb{x}^\top\mbox{diag}(\bds{\pi}_c)^{-1}\mb{y}.
\end{align*}
The above inner products are well-defined because the Perron vectors,~$\bds{\pi}_r$ and~$\bds{\pi}_c$, are positive and thus respectively induce a weighted Euclidean norm as follows:~$\forall \mb{x}\in\mathbb{R}^n$,
\begin{align*}
&\left\|\mb{x}\right\|_{\bds{\pi}_r}\triangleq\sqrt{[\bds{\pi}_r]_1x_1^2+\dots+[\bds{\pi}_r]_nx_n^2} = \left\|\mbox{diag}(\sqrt{\bds{\pi}_r})\mb{x}\right\|_2,\\
&\left\|\mb{x}\right\|_{\bds{\pi}_c}\triangleq\sqrt{\frac{x_1^2}{[\bds{\pi}_c]_1}+\dots+\frac{x_n^2}{[\bds{\pi}_c]_n}}=\left\|\mbox{diag}(\sqrt{\bds{\pi}_c})^{-1}\mb{x}\right\|_2.
\end{align*}
We denote~$\mn{\cdot}_{\bds{\pi}_r}$ and~$\mn{\cdot}_{\bds{\pi}_c}$ as the matrix norms induced by~$\left\|\cdot\right\|_{\bds{\pi}_r}$ and~$\left\|\cdot\right\|_{\bds{\pi}_c}$, respectively, i.e.,~$\forall X\in\mathbb{R}^{n\times n}$, see~\cite{matrix},
\begin{align}
&\mn{X}_{\bds{\pi}_r}=\mn{\mbox{diag}(\sqrt{\bds{\pi}_r})X\:\mbox{diag}(\sqrt{\bds{\pi}_r})^{-1}}_2, \label{norm1}\\
&\mn{X}_{\bds{\pi}_c}=\mn{\mbox{diag}(\sqrt{\bds{\pi}_c})^{-1}X\:\mbox{diag}(\sqrt{\bds{\pi}_c}}_2\label{norm2}. 
\end{align}
It can be verified that the corresponding norm equivalence relationships between~$\|\cdot\|_2$,~$\|\cdot\|_{\bds{\pi}_r}$, and~$\|\cdot\|_{\bds{\pi}_c}$ are given by
\begin{align*}
&\|\cdot\|_{\bds{\pi}_r}\leq\ol{\bds{\pi}_r}^{0.5}\|\cdot\|_2,\qquad&&\|\cdot\|_2\leq\ol{\bds{\pi}_c}^{0.5}\|\cdot\|_{\bds{\pi}_c},\\ 
&\|\cdot\|_{\bds{\pi}_c}\leq\ul{\bds{\pi}_c}^{-0.5}\|\cdot\|_2,\qquad&&\|\cdot\|_2\leq\ul{\bds{\pi}_r}^{-0.5}\|\cdot\|_{\bds{\pi}_r}.
\end{align*}

\noindent We next  establish the contraction of the~$A$ and~$B$ matrices with the help of the above arguments.
\begin{lem}\label{A_B}
Let Assumption~\ref{asp3} hold. Consider the weight matrices~$A,B$ in~\eqref{SABv}. We have:~$\forall \mb{x}\in\mathbb{R}^n$,
\begin{align}
\left\|A\mb{x}-A_\infty\mb{x}\right\|_{\bds{\pi}_r}\leq\sigma_A\left\|\mb{x}-A_\infty\mb{x}\right\|_{\bds{\pi}_r},\label{A}\\
\left\|B\mb{x}-B_\infty\mb{x}\right\|_{\bds{\pi}_c}\leq\sigma_B\left\|\mb{x}-B_\infty\mb{x}\right\|_{\bds{\pi}_c}\label{B},
\end{align}
with~$\sigma_A\triangleq\mn{A-A_\infty}_{\bds{\pi}_r}\!<\!1$ and~$\sigma_B\triangleq\mn{B-B_\infty}_{\bds{\pi}_c}\!<\!1$.
\end{lem}
The proof of Lemma~\ref{A_B} is available in the Appendix. It can be further verified that
\begin{align*}
    &\sigma_A=\sigma_2\Big(\mbox{diag}(\sqrt{\bds{\pi}_r})A\mbox{diag}(\sqrt{\bds{\pi}_r})^{-1}\Big),\\
    &\sigma_B=\sigma_2\Big(\mbox{diag}(\sqrt{\bds{\pi}_c})^{-1}B\mbox{diag}(\sqrt{\bds{\pi}_c}) \Big),\\
	&\mn{A}_{\bds{\pi}_r} =\mn{A_\infty}_{\bds{\pi}_r}=\mn{I_n-A_\infty}_{\bds{\pi}_r} = 1,\\
	&\mn{B}_{\bds{\pi}_c} =\mn{B_\infty}_{\bds{\pi}_c}=\mn{I_n-B_\infty}_{\bds{\pi}_c} = 1,
\end{align*}
where~$\sigma_2(\cdot)$ is the second largest singular value of a matrix.

In the following, Lemma~\ref{SG} provides some simple results on the stochastic gradients, Lemma~\ref{lsm} uses the~$l$-smoothness of the cost functions, while Lemmas~\ref{cvx_stan} and~\ref{rho} are standard in convex optimization and matrix analysis. To present these results, we define three quantities:
$$\ol{\mb{y}}_k \triangleq \frac{1}{n}\mb{1}_n^\top\mb{y}_k,~~\mb{h}(\mb{x}_k)\triangleq\frac{1}{n}\mb{1}_n^\top\nabla\mb{f}(\mb{x}_k),~~\widehat{\mb{x}}_k\triangleq\bds{\pi}_r^\top\mb{x}_k,~~ \ol{\mb{g}}(\mb{x}_{k},\bds{\xi}_k)\triangleq \frac{1}{n}\mb{1}_n^\top\mb{g}(\mb{x}_{k},\bds{\xi}_k),$$ where~$\nabla\mb{f}(\mb{x}_k)\triangleq[\nabla f_1(\mb{x}_k^1)^\top,\dots,\nabla f_n(\mb{x}_k^n)^\top]^\top$. The following statements use standard arguments and their formal proofs are omitted due to space limitations. Similar results can be found in~\cite{harness,AB,pu2018distributed}.

\begin{lem}\label{SG}
Consider the iterates~$\{\mb{y}_k\}_{k\geq 0}$ generated by~$\mc{S}$-$\mc{AB}$ in \eqref{SABv2} and let Assumptions~\ref{asp2} and~\ref{asp4} hold. Then the following hold,~$\forall k\geq0$:
\begin{enumerate}[(1)]
\item $\ol{\mb{y}}_k =  \ol{\mb{g}}(\mb{x}_{k},\bds{\xi}_k)$
\item $\mathbb{E}\left[\ol{\mb{y}}_k|\mc{F}_k\right] = \mb{h}(\mb{x}_k)$
\item $\mathbb{E}\left[\left\|\ol{\mb{y}}_k-\mb{h}(\mb{x}_k)\right\|_2^2\big|\mc{F}_k\right]\leq\frac{\sigma^2}{n}$ \label{sg}
\end{enumerate}
\end{lem}
%

\begin{lem}\label{lsm}
Consider the iterates~$\{\mb{x}_k\}_{k\geq 0}$ generated by the~$\mc{S}$-$\mc{AB}$ algorithm in \eqref{SAB} and let Assumptions~\ref{asp2} hold. Then the following holds,~$\forall k\geq\!0$:
$$\|\mb{h}(\mb{x}_k)-\nabla F(\widehat{\mb{x}}_k)\|_2\leq\frac{l}{\sqrt{n}}\|\mb{x}_k-\mb{1}_n\widehat{\mb{x}}_k\|_2.$$
\end{lem}

\begin{lem}[\cite{nesterov2013introductory}]\label{cvx_stan}
Let Assumptions~\ref{asp1} and~
\ref{asp2} hold. Recall that the global objective function,~$F$, is~$\mu$-strongly-convex and~$l$-smooth. If~$0<\alpha<\frac{1}{l}$, we have:~$\forall \mb{x}\in\mathbb{R}^p$,
\begin{equation*}
\|\mb{x}-\nabla F(\mb{x})-\mb{x}^*\|_2\leq(1-\alpha\mu)\|\mb{x}-\mb{x}^*\|_2,
\end{equation*}
where~$\mb{x}^*$ is the global minimizer of~$F$.
\end{lem}

\begin{lem}[\cite{matrix}]\label{rho}
Let~$X\in\mathbb{R}^{n\times n}$ be non-negative and~$\mb{x}\in\mathbb{R}^{n}$ be a positive vector. If~$X\mb{x}<\omega\mb{x}$ with~$\omega>0$, then~$\rho(X)<\omega$. 
\end{lem}

\section{Convergence analysis}\label{conv}
In this section, we analyze the~$\mc{S}$-$\mc{AB}$ algorithm and establish its convergence properties for which we present Lemmas~\ref{lem_y}-\ref{lem_iii}. The proofs for these lemmas are provided in the Appendix. First, in Lemma~\ref{lem_y}, we bound~$\|\mb{y}_k\|_2^2$. 
\begin{lem}\label{lem_y}
Let Assumptions~\ref{asp1}-\ref{asp4} hold. Then the iterates~$\{\mb{y}_k\}_{k\geq 0}$ in~\eqref{SABv} follow:
\begin{align}
\mathbb{E}\left[\left\|\mb{y}_k\right\|_2^2\big|\mc{F}_k\right] \leq&~\frac{4n\|\bds{\pi}_c\|_2^2l^2}{\ul{\bds{\pi}_r}}\left\|\mb{x}_k-\mb{1}_n\widehat{\mb{x}}_k\right\|_{\bds{\pi}_r}^2
+ 4n^2\|\bds{\pi}_c\|_2^2l^2\left\|\widehat{\mb{x}}_k-\mb{x}^*\right\|^2_2  \nonumber\\
&+ 4\ol{\bds{\pi}_c}\mathbb{E}\left[\left\|\mb{y}_k - B_\infty\mb{y}_k\right\|^2_{\bds{\pi}_c}\big|\mc{F}_k\right]
+ 4n\|\bds{\pi}_c\|_2^2\sigma^2. 
\end{align}
\end{lem}

Next in Lemmas~\ref{lem_i}-\ref{lem_iii}, we bound the following three quantities in expectation, conditioned on the $\sigma$-algebra~$\mc{F}_k$: (i)~$\|\mb{x}_{k+1}-\mb{1}_n\widehat{\mb{x}}_{k+1}\|^2_{\bds{\pi}_r}$, the consensus error in the network; (ii)~$\|\widehat{\mb{x}}_{k+1} - \mb{x}^*\|^2_2$, the optimality gap; and, (iii)~$\|\mb{y}_{k+1}-B_\infty\mb{y}_{k+1}\|_{\bds{\pi}_c}^2$, the gradient tracking error. We then show that the norm of a vector composed of these three quantities converges linearly to a ball around the optimal when the step-size~$\alpha$ is fixed and sufficiently small. The first lemma below is on the consensus error. 
\begin{lem}\label{lem_i}
Let Assumption~\ref{asp3} hold. Then the consensus error in the network follows:~$\forall k\geq0$,
\begin{align}
\mbb{E}&\left[\left\|\mb{x}_{k+1}-\mb{1}_n\widehat{\mb{x}}_{k+1}\right\|^2_{\bds{\pi}_r}\big|\mc{F}_k\right] \leq\frac{1+\sigma_A^2}{2}~\left\|\mb{x}_k -\mb{1}_n\widehat{\mb{x}}_k\right\|^2_{\bds{\pi}_r}+ \frac{2\ol{\bds{\pi}_r}\alpha^2}{1-\sigma_A^2}~\mbb{E}\left[\left\|\mb{y}_k\right\|^2_2\big|\mc{F}_k\right].
\end{align}
\end{lem}
The next lemma is on the optimality gap. 
\begin{lem}\label{lem_ii}
Let Assumptions~\ref{asp1}-\ref{asp4} hold. If~$0<\alpha<\tfrac{1}{n\bds{\pi}_r^\top\bds{\pi}_cl}$, the optimality gap in the network follows:~$\forall k\geq0$,
\begin{align}
\mathbb{E}\left[\left\|\widehat{\mb{x}}_{k+1} - \mb{x}^*\right\|^2_2\big|\mc{F}_k\right]\leq  
&~ \frac{3\alpha\bds{\pi}_r^\top\bds{\pi}_c l^2}{\mu\ul{\bds{\pi}_r}}\left\|\mb{x}_k-\mb{1}_n\widehat{\mb{x}}_k\right\|_{\bds{\pi}_r}^2 \nonumber\\
&+\left(1-\frac{\mu n\bds{\pi}_r^\top\bds{\pi}_c\alpha}{2}\right)\left\|\widehat{\mb{x}}_{k}-\mb{x}^*\right\|^2_2
 \nonumber\\
&+\frac{3\alpha\|\bds{\pi}_r\|_2^2\ol{\bds{\pi}_c}}{\mu n\bds{\pi}_r^\top\bds{\pi}_c }~\mathbb{E}\left[\left\|\mb{y}_k-B_\infty\mb{y}_k\right\|_{\bds{\pi}_c}^2|\mc{F}_k\right]\nonumber\\
&+\frac{3\left(\bds{\pi}_r^\top\bds{\pi}_c\right)^2n\alpha ^2\sigma^2}{2}.
\end{align}
\end{lem}

\noindent Finally, we quantify the gradient tracking error. 
\begin{lem}\label{lem_iii}
Let Assumptions~\ref{asp2}-\ref{asp4} hold. The gradient tracking error follows:
\begin{align}
\mathbb{E}\left[\left\|\mb{y}_{k+1}-B_\infty\mb{y}_{k+1}\right\|_{\bds{\pi}_c}^2|\mc{F}_k\right] 
\leq&~
\frac{16l^2}{\ul{\bds{\pi}_r}\:\ul{\bds{\pi}_c}(1-\sigma_B^2)}\left\|\mb{x}_{k}-\mb{1}_n\widehat{\mb{x}}_{k}\right\|_{\bds{\pi}_r}^2
\nonumber\\
&+
\frac{1+\sigma_B^2}{2}~\mathbb{E}\left[\left\|\mb{y}_{k}-B_\infty\mb{y}_{k}\right\|_{\bds{\pi}_c}^2|\mc{F}_k\right]
\nonumber\\
&+
\frac{4l^2\alpha^2}{\ul{\bds{\pi}_c}(1-\sigma_B^2)}~\mbb{E}\left[\left\|\mb{y}_k\right\|_{2}^2 | \mc{F}_k\right]
+\frac{2nl\sigma^2\alpha}{\ul{\bds{\pi}_c}}+\frac{4n\sigma^2}{\ul{\bds{\pi}_c}}.
\end{align}
\end{lem}

\vspace{0.5cm}
With the help of the above lemmas, we define a vector,~$\mb{t}_k\in\mathbb{R}^3$, i.e.,
\begin{align*}
\mb{t}_k = \left[
\begin{array}{c}
\mathbb{E}\left[\|\mb{x}_{k}-\mb{1}_n\widehat{\mb{x}}_{k}\|^2_{\bds{\pi}_r}\right]\\
\mathbb{E}\left[\|\widehat{\mb{x}}_{k} - \mb{x}^*\|^2_2\right]\\
\mathbb{E}\left[\|\mb{y}_{k}-B_\infty\mb{y}_{k}\|_{\bds{\pi}_c}^2\right].
\end{array}
\right]
\end{align*}
By substituting the bound on~$\mbb{E}[\|\mb{y}_k\|_2^2|\mc{F}_k]$ from Lemma~\ref{lem_y} in Lemmas~\ref{lem_i}-\ref{lem_iii}, and taking the full expectation of both sides, it can be verified that~$\mb{t}_k$ follows the dynamical system below.
\begin{align}
\mb{t}_{k+1} \leq& \left[
\begin{array}{ccc}
\tfrac{1+\sigma_A^2}{2}+a_1\alpha^2&a_2\alpha^2&a_3\alpha^2\\
a_4\alpha&1-a_5\alpha&a_6\alpha\\
a_7+a_8\alpha^2&a_9\alpha^2&\frac{1+\sigma_B^2}{2}+a_{10}\alpha^2
\end{array}\right]\mb{t}_k
+ \left[
\begin{array}{c}
b_1\alpha^2\\
b_2\alpha^2\\
b_3+b_4\alpha+b_5\alpha^2
\end{array}\right],\nonumber\\
\triangleq&~G_\alpha\mb{t}_k + \mb{b}_\alpha,
\end{align}
where the constants are given by
\begin{align*}
\begin{array}{ll}
a_1 = \tfrac{8\ol{\bds{\pi}_r}n\|\bds{\pi}_c\|_2^2l^2}{\ul{\bds{\pi}_r}\left(1-\sigma_A^2\right)}, \qquad&a_8 = 
\tfrac{16nl^4\|\bds{\pi_c}\|_2^2}{\ul{\bds{\pi_r}}\:\ul{\bds{\pi_c}}(1-\sigma_B^2)}, \\
a_2 = \tfrac{8\ol{\bds{\pi}_r}n^2\|\bds{\pi}_c\|_2^2l^2}{1-\sigma_A^2},\qquad&a_9 = \tfrac{16n^2l^4\|\bds{\pi_c}\|_2^2}{\ul{\bds{\pi_c}}(1-\sigma_B^2)},\\
a_3 = \tfrac{8\ol{\bds{\pi_r}}\:\ol{\bds{\pi_c}}}{1-\sigma_A^2},\qquad&a_{10}=\tfrac{16\ol{\bds{\pi_c}}l^2}{\ul{\bds{\pi_c}}\left(1-\sigma_B^2\right)},\\
a_4 = \frac{3\bds{\pi}_r^\top\bds{\pi}_c l^2}{\mu\ul{\bds{\pi}_r}},\qquad&b_1 = \tfrac{8\ol{\bds{\pi}_r}n\|\bds{\pi}_c\|_2^2\sigma^2}{1-\sigma_A^2},\\
a_5 = \frac{\mu n\bds{\pi}_r^\top\bds{\pi}_c}{2},\qquad&b_2 = \frac{3\left(\bds{\pi}_r^\top\bds{\pi}_c\right)^2n\sigma^2}{2},\\
a_6 = \frac{3\|\bds{\pi}_r\|_2^2\ol{\bds{\pi}_c}}{\mu n\bds{\pi}_r^\top\bds{\pi}_c},\qquad&b_3 =  \tfrac{4n\sigma^2}{\ul{\bds{\pi}_c}},\\
a_7 =\tfrac{16l^2}{\ul{\bds{\pi}_c}\:\ul{\bds{\pi}_r}(1-\sigma_B^2)},\qquad&b_4 = 
\tfrac{2nl\sigma^2}{\ul{\bds{\pi}_c}}, \\
b_5 =  
\tfrac{16nl^2\|\bds{\pi_c}\|_2^2\sigma^2}{\ul{\bds{\pi_r}}\:\ul{\bds{\pi_c}}(1-\sigma_B^2)}.
\end{array}
\end{align*}

\section{Main Result}\label{main}
In this section, we analyze the inequality on~$\mb{t}_k$ to establish the convergence of~$\mc{S}$-$\mc{AB}$.

\begin{theorem}\label{thm1}
Consider the~$\mc{S}$-$\mc{AB}$ algorithm in \eqref{SAB} and let Assumptions \ref{asp1}-\ref{asp4} hold. Suppose the step-size~$\alpha$ satisfies the following the condition:
\begin{align*}
0<\alpha
<\min\left\{\frac{1}{ln\bds{\pi}_r^\top\bds{\pi}_c},
 \frac{\left(1-\sigma_A^2\right)\sqrt{\left(1-\sigma_B^2\right)}\left(\bds{\pi}_r^\top\bds{\pi}_c\right)}{l\kappa\left\|\bds{\pi}_r\right\|_2\left\|\bds{\pi}_c\right\|_2\sqrt{384h_rh_c(n\left\|\bds{\pi}_c\right\|_2^2+64)}}
,
\frac{\left(1-\sigma_B^2\right)\left(\bds{\pi}_r^\top\bds{\pi}_c\right)}{l\kappa\left\|\bds{\pi}_r\right\|_2\left\|\bds{\pi}_c\right\|_2\sqrt{24\left(n\left\|\bds{\pi}_c\right\|_2^2+48h_c\right)}}
\right\},
\end{align*}
where~$h_r = \ol{\bds{\pi}_r}/\ul{\bds{\pi}_r}$,~$h_c = \ol{\bds{\pi}_c}/\ul{\bds{\pi}_c}$ and~$\kappa = l/\mu$. 
Then,  $\rho(G_\alpha)<1$, the vector $(I_3-G_\alpha)^{-1}\mb{b}_\alpha$ has non-negative components, and we have that
\begin{align*}
&\lim\sup_{k\rightarrow\infty}\mb{t}_k \leq (I_3-G_\alpha)^{-1}\mb{b}_\alpha \\
&\lim\sup_{k\rightarrow\infty}\mathbb{E}\left[\|\mb{x}_{k}-\mb{1}_n\widehat{\mb{x}}_{k}\|^2_{\bds{\pi}_r}\right]\leq \left[(I_3-G_\alpha)^{-1}\mb{b}_\alpha\right]_1\\
&\lim\sup_{k\rightarrow\infty}\mathbb{E}\left[\|\widehat{\mb{x}}_{k} - \mb{x}^*\|^2_2\right] \leq \left[(I_3-G_\alpha)^{-1}\mb{b}_\alpha\right]_2,
\end{align*}where the above convergence is geometric with exponent~$\rho(G_\alpha)$.
\end{theorem}

\begin{proof}
The goal is to find the range of~$\alpha$ such that~$\rho(G_\alpha)<1$. In the light of Lemma~\ref{rho}, it suffices to solve for the range of~$\alpha$ such that~$G_{\alpha}\bds{\delta}<\bds{\delta}$ holds for some positive vector~$\bds{\delta}=[\delta_1,\delta_2,\delta_3]^\top$. We now expand this element-wise matrix inequality as follows:
\begin{align*}
\left(\tfrac{1+\sigma_A^2}{2}+a_1\alpha^2\right)\delta_1+a_2\alpha^2\delta_2+a_3\alpha^2\delta_3&<\delta_1  \\
a_4\alpha\delta_1+(1-a_5\alpha)\delta_2+a_6\alpha\delta_3&<\delta_2 \\
\left(a_7+a_8\alpha^2\right)\delta_1+a_9\alpha^2\delta_2+\left(\tfrac{1+\sigma_B^2}{2}+a_{10}\alpha^2\right)\delta_3&<\delta_3
\end{align*}
which can be reformulated as:
\begin{align}
\left(a_1\delta_1+a_2\delta_2+a_3\delta_3\right)\alpha^2&<\tfrac{1-\sigma_A^2}{2}\delta_1 \label{i1}\\
a_4\delta_1\alpha-a_5\delta_2\alpha+a_6\delta_3\alpha&<0
 \label{i2}\\
\left(a_8\delta_1+a_9\delta_2+a_{10}\delta_3\right)\alpha^2&<\tfrac{1-\sigma_B^2}{2}\delta_3-a_7\delta_1 \label{i3}.
\end{align}
We first note that in order for~\eqref{i2} to hold,~$\delta_1,\delta_2,\delta_3$ suffice to satisfy:
\begin{equation}
\delta_2 > \frac{a_4\delta_1+a_6\delta_3}{a_5} \label{d2}
\end{equation}
Next, for the right handside of~\eqref{i3} to be positive,~$\delta_1,\delta_3$ suffice to satisfy:
\begin{equation}
0<\delta_1 < \frac{1-\sigma_B^2}{2a_7}\delta_3 \label{d1}
\end{equation}
In order to obtain an explicit upper bound on the step-size,~$\ol\alpha$, such that~\eqref{i1}-\eqref{i3} hold when~$0<\alpha<\ol\alpha$, we set~$\delta_1,\delta_2,\delta_3$ as the following which satisfies~\eqref{d2} and~\eqref{d1}:
\begin{align}
\delta_1 = 1, \delta_2 = \frac{2}{a_5}\left(a_4+\frac{4a_6a_7}{1-\sigma_B^2}\right), \delta_3 = \frac{4a_7}{1-\sigma_B^2}.  \label{d_val} 
\end{align}
Then we plug the values of~$\delta_1,\delta_2,\delta_3$ in~\eqref{d_val} into~\eqref{i1} and~\eqref{i2} to solve an upper bound on the step-size,~$\alpha$. From~\eqref{i1}, we get:
\begin{align}
\left(a_1+\frac{2a_2}{a_5}\left(a_4+\frac{4a_6a_7}{1-\sigma_B^2}\right)+\frac{4a_3a_7}{1-\sigma_B^2}\right)\alpha^2 <
\frac{1-\sigma_A^2}{2} \label{a1_0}
\end{align}
One can further verify that:
\begin{align*}
a_1+\frac{2a_2}{a_5}\left(a_4+\frac{4a_6a_7}{1-\sigma_B^2}\right)+\frac{4a_3a_7}{1-\sigma_B^2}
< \frac{192h_rh_c\left\|\bds{\pi}_r\right\|_2^2\left\|\bds{\pi}_c\right\|_2^2l^2\kappa^2(64+n\left\|\bds{\pi}_c\right\|_2^2)}{\left(1-\sigma_A^2\right)\left(1-\sigma_B^2\right)\left(\bds{\pi}_r^\top\bds{\pi}_c\right)^2}
\end{align*}
We use the above inequality in~\eqref{a1_0} to obtain an upper bound on~$\alpha$:
\begin{align}
\alpha < \frac{\left(1-\sigma_A^2\right)\left(\bds{\pi}_r^\top\bds{\pi}_c\right)}{l\kappa\left\|\bds{\pi}_r\right\|_2\left\|\bds{\pi}_c\right\|_2}\sqrt{
\frac{\left(1-\sigma_B^2\right)}{384h_rh_c(64+n\left\|\bds{\pi}_c\right\|_2^2)}
}
\label{a1}
\end{align}
Similarly, from~\eqref{i3}, we get:
\begin{align}
\left(
\frac{a_8}{a_7} + \frac{2a_9}{a_5a_7}\left(a_4+\frac{4a_6a_7}{1-\sigma_B^2}\right) + \frac{4a_{10}}{1-\sigma_B^2}
\right)\alpha^2 < 1 \label{a2_0}
\end{align}
One can then verify that:
\begin{align*}
\frac{a_8}{a_7} + \frac{2a_9}{a_5a_7}\left(a_4+\frac{4a_6a_7}{1-\sigma_B^2}\right) + \frac{4a_{10}}{1-\sigma_B^2}
<
\frac{24l^2\kappa^2\left\|\bds{\pi}_r\right\|_2^2\left\|\bds{\pi}_c\right\|_2^2\left(n\left\|\bds{\pi}_c\right\|_2^2+48h_c\right)}{\left(1-\sigma_B^2\right)^2\left(\bds{\pi}_r^\top\bds{\pi}_c\right)^2}
\end{align*}
Using the above inequality in~\eqref{a2_0} obtains another upper bound on~$\alpha$:
\begin{align}
\alpha < \frac{\left(1-\sigma_B^2\right)\left(\bds{\pi}_r^\top\bds{\pi}_c\right)}{l\kappa\left\|\bds{\pi}_r\right\|_2\left\|\bds{\pi}_c\right\|_2\sqrt{24\left(n\left\|\bds{\pi}_c\right\|_2^2+48h_c\right)}}
\label{a2}
\end{align}
We complete the proof by combining~\eqref{a1},~\eqref{a2} and the requirement from Lemma~\ref{lem_ii} that~$0<\alpha<\frac{1}{n\bds{\pi}_r^\top\bds{\pi}_cl}$.
\end{proof}
It is important to note that the error bounds in Theorem~\ref{thm1} go to zero as the step-size gets smaller and the variance on the gradient noise decreases.

\section{Numerical Experiments}\label{sims}
In this section, we illustrate the~$\mc{S}$-$\mc{AB}$ algorithm and its convergence properties. We demonstrate the results on a directed graph generated using nearest neighbor rules with~$n=20$ agents. The particular graph for the experiments is shown in Fig.~\ref{gr1} (left) to provide a sense of connectivity. We choose a logistic regression problem to classify around~$12,000$ images of two digits,~$7$ and~$8$, labeled as~$y_{ij}=+1$ or~$-1$,  from the MNIST dataset~\cite{mnist}. Each image,~$\mb{c}_{ij}$, is a~$785$-dimensional vector and the total images are divided among the agents such that each agent has~$m_i=600$ images. Because privacy and communication restrictions, the agents do not share their local batches (local training images) with each other. In order to use the entire data set for training, the network of agents cooperatively solves the following distributed logistic regression problem:
{\begin{align*}
\underset{\mb{w},b}{\operatorname{min}}~F(\mb{w},b)
=\sum_{i=1}^n\sum_{j=1}^{m_i}\ln\left[1+e^{-\left(\mb{w}^\top\mb{c}_{ij}+b\right)y_{ij}}\right]+\frac{\lambda}{2}\|\mb{w}\|_2^2\nonumber,
\end{align*}}where the private function at each agent,~$i$, is given by:
{\[
f_i(\mb{w},b)=\sum_{j=1}^{m_i}\ln\left[1+e^{-\left(\mb{w}^\top\mb{c}_{ij}+b\right)y_{ij}}\right] \nonumber\\
+\frac{\lambda}{2n}\|\mb{w}\|_2^2.
\]}

We show the performance of this classification problem over centralized and distributed methods. Centralized gradient descent (CGD) uses the entire batch, i.e., it computes~$12,000$ gradients at each iteration, whereas centralized stochastic gradient descent (C-SGD) uses only one data point at each iteration that is uniformly sampled from the entire batch. For the distributed algorithms, we show the performance of non-stochastic~$\mc{AB}$, where each agent uses its entire local batch, i.e.,~$600$ labeled data points. Whereas, for the implementation of~$\mc{S}$-$\mc{AB}$, each agent uniformly chooses one data point from its local batch. For testing, we use~$2000$ additional images that were not used for training. The residuals are shown in Fig.~\ref{gr1} (right) while the training and testing accuracy is shown in Fig.~\ref{comp1}. In the performance figures, the horizontal axis represents the number of epochs where each epoch represents computations on the entire batch. Clearly,~$\mc{S}$-$\mc{AB}$ has a better performance when compared to~$\mc{AB}$ in~\cite{AB} as expected from the performance of their centralized counterparts,~C-SGD and CGD.

\begin{figure}[!h]
\centering
\includegraphics[width=2in]{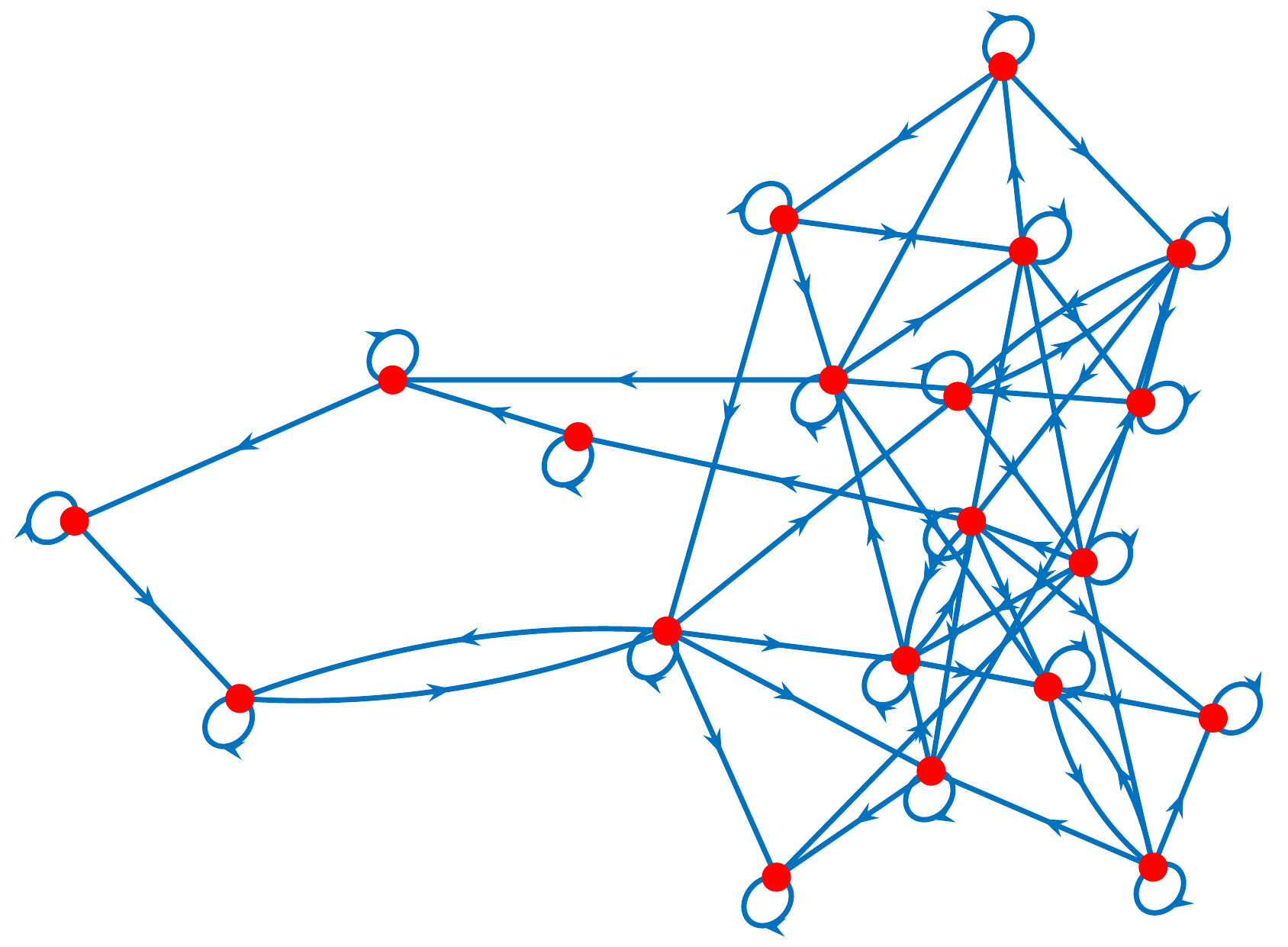}
\hspace{1cm}
\includegraphics[width=2.5in]{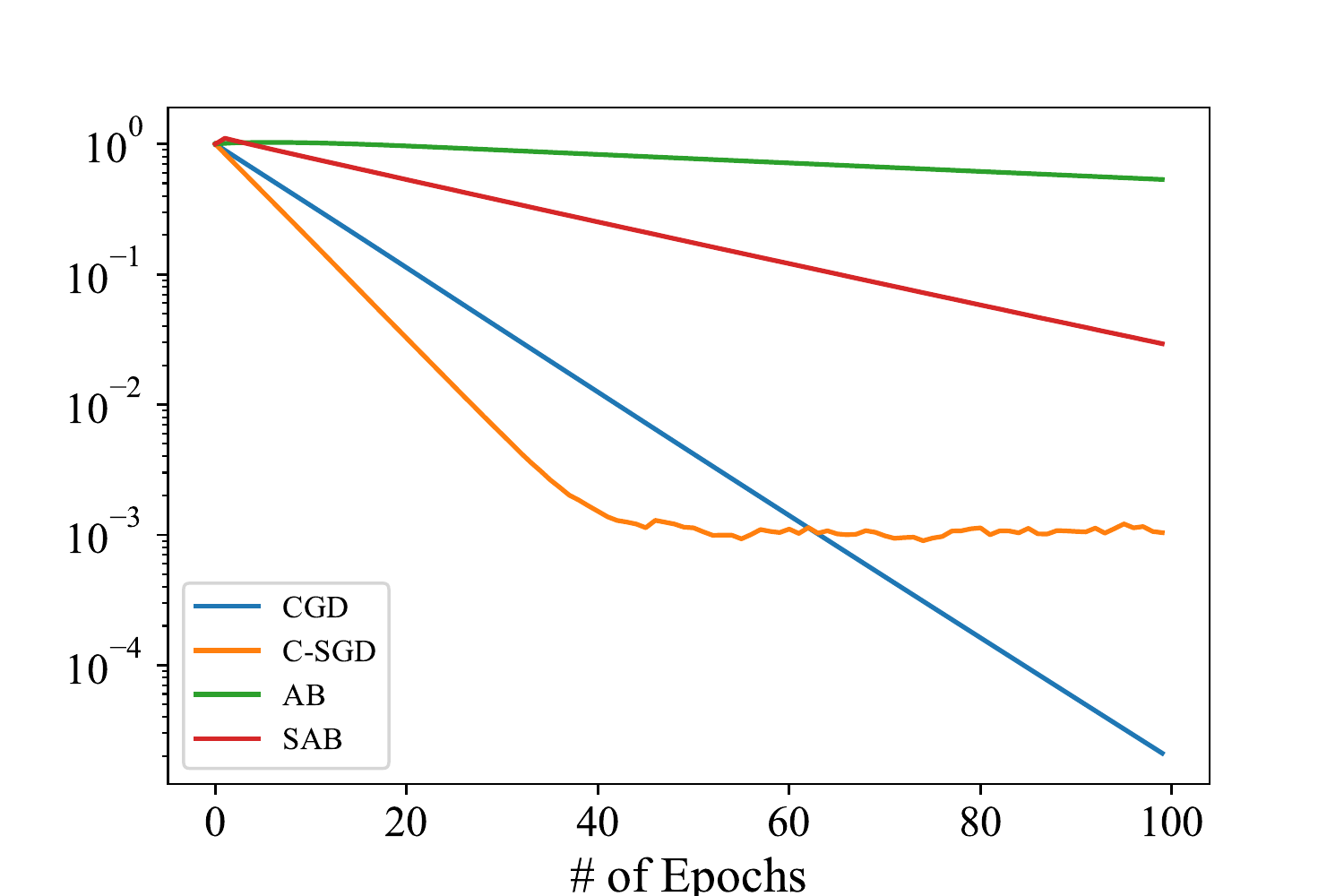}
\caption{(Left) Strongly-connected directed graph. (Right) Residuals.}
\label{gr1}
\end{figure}

\begin{figure}[!h]
\centering
\subfigure{\includegraphics[width=2.5in]{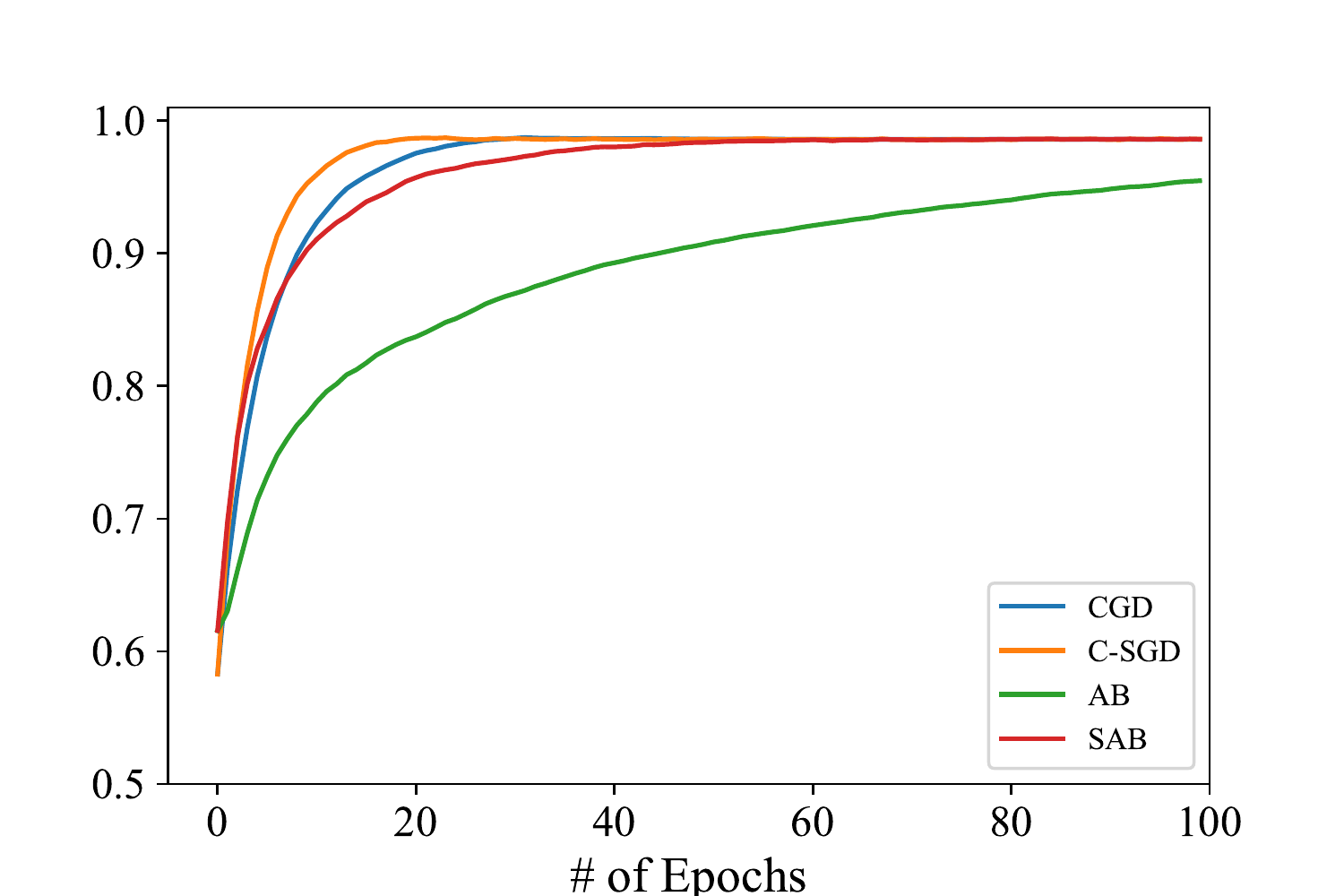}}
\hspace{1cm}
\subfigure{\includegraphics[width=2.5in]{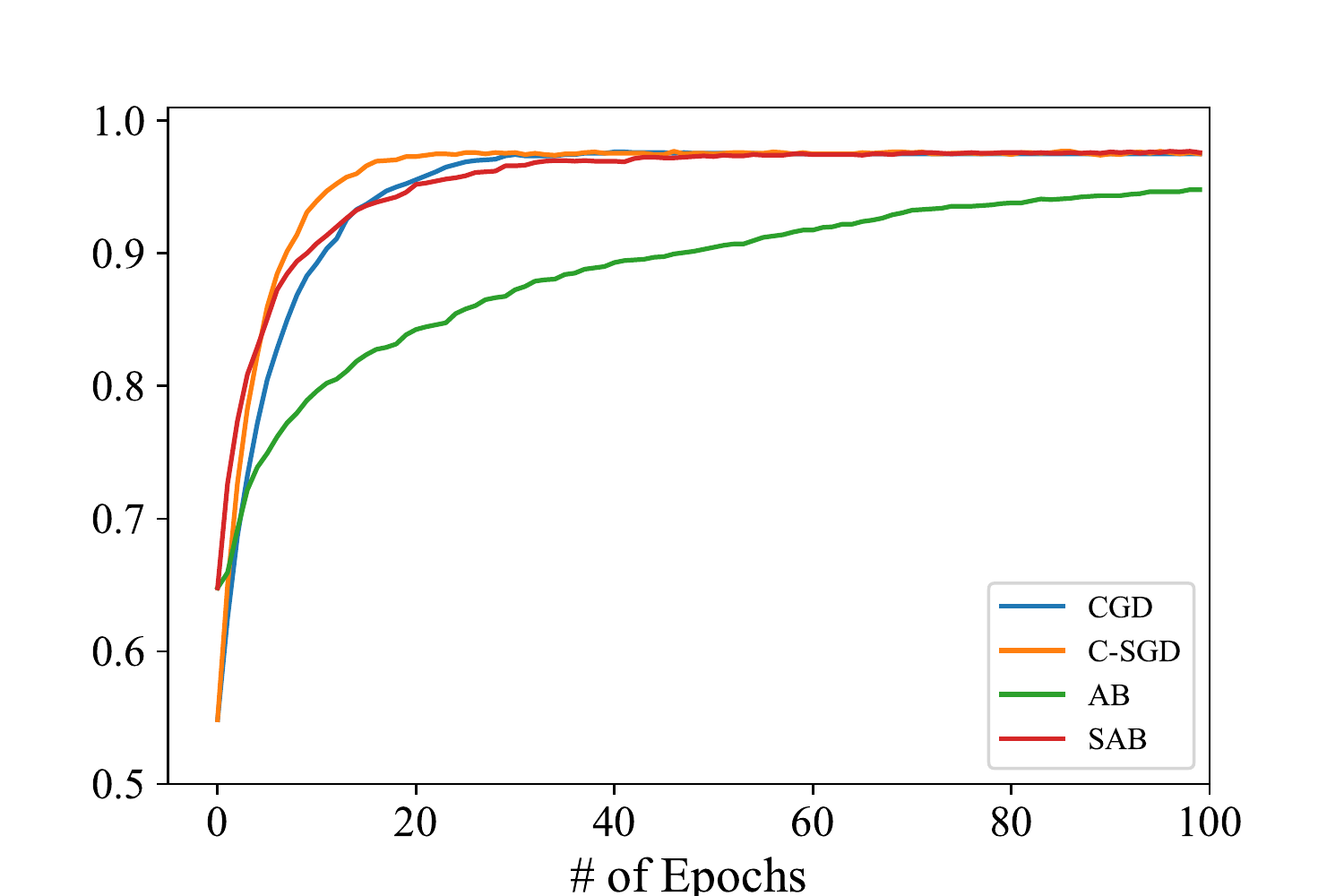}}
\caption{(Left) Training accuracy. (Right) Test accuracy.}
\label{comp1}
\end{figure}

\section{Conclusions}\label{conc}
In this paper, we have presented a stochastic gradient descent algorithm,~$\mc{S}$-$\mc{AB}$, over arbitrary strongly-connected graphs. In this setup, the data is distributed over agents and each agent uniformly samples a data point (from its local batch) at each iteration of the algorithm to implement the stochastic~$\mc{S}$-$\mc{AB}$ algorithm. To cope with general directed communication graphs and potential lack of doubly-stochastic weights,~$\mc{S}$-$\mc{AB}$ employs a two-phase update with row- and column-stochastic weights.  We have shown that under a sufficiently small constant step-size,~$\mc{S}$-$\mc{AB}$ converges linearly to a neighborhood of the global minimizer when the local cost functions are smooth and strongly-convex. We have presented numerical simulations based on real-world datasets to illustrate the theoretical results.

\bibliographystyle{IEEEbib}
\bibliography{sample1.bib}

\appendix
\section*{Proof of Lemma~\ref{A_B}}
\begin{proof}
We start with the proof of~\eqref{A}. Note that~$A\mb{x}-A_\infty\mb{x} = (A-A_\infty)(\mb{x}-A_\infty\mb{x})$ that leads to
\begin{equation*}
\left\|A\mb{x}-A_\infty\mb{x}\right\|_{\bds{\pi}_r} \leq \mn{A-A_\infty}_{\bds{\pi}_r}\left\|\mb{x}-A_\infty\mb{x}\right\|_{\bds{\pi}_r}.
\end{equation*}
By the definition of~$\mn{\cdot}_{\bds{\pi}_r}$ in~\eqref{norm1}, we have
\begin{align*}
\mn{A-A_\infty}_{\bds{\pi}_r} &=  \mn{\mbox{diag}(\sqrt{\bds{\pi}_r})(A-A_\infty)\mbox{diag}(\sqrt{\bds{\pi}_r})^{-1}}_2\triangleq\sqrt{\lambda_{\max} (J)},
\end{align*}
where~$J=\mbox{diag}(\sqrt{\bds{\pi}_r})^{-1}(A-A_\infty)^\top\mbox{diag}(\bds{\pi}_r)(A-A_\infty)\mbox{diag}(\sqrt{\bds{\pi}_r})^{-1}$ and~$\lambda_{\max}(\cdot)$ denotes the largest eigenvalue of the matrix. What we need to show is that~$\rho(J)<1$. Expanding~$J$, we~get
\begin{align*}
J =&~\mbox{diag}(\sqrt{\bds{\pi}_r})^{-1}A^\top\mbox{diag}(\bds{\pi}_r)A\mbox{diag}(\sqrt{\bds{\pi}_r})^{-1} 
- \mbox{diag}(\sqrt{\bds{\pi}_r})^{-1}A_\infty^\top\mbox{diag}(\bds{\pi}_r)A\mbox{diag}(\sqrt{\bds{\pi}_r})^{-1} \\
&- \mbox{diag}(\sqrt{\bds{\pi}_r})^{-1}A^\top\mbox{diag}(\bds{\pi}_r)A_\infty\mbox{diag}(\sqrt{\bds{\pi}_r})^{-1}
+ \mbox{diag}(\sqrt{\bds{\pi}_r})^{-1}A_\infty^\top\mbox{diag}(\bds{\pi}_r)A_\infty\mbox{diag}(\sqrt{\bds{\pi}_r})^{-1},\\
\triangleq&~J_1 - J_2 - J_3 + J_4.
\end{align*}
With the fact that~$A_\infty=\mb{1}_n\bds{\pi}_r^\top$, it can be verified that~$J_2=J_3=J_4=\sqrt{\bds{\pi}}_r\sqrt{\bds{\pi}}_r^\top$ and thus~$J = J_1 - \sqrt{\bds{\pi}}_r\sqrt{\bds{\pi}}_r^\top.$~Furthermore,~$J_1\sqrt{\bds{\pi_r}} = \sqrt{\bds{\pi_r}}$, and~$\sqrt{\bds{\pi_r}}^\top J_1 = \sqrt{\bds{\pi_r}}^\top.$ Since~$J_1$ is primitive, by Perron-Frobenius theorem~\cite{matrix}, we have~$\rho(J) = \rho(J_1 - \sqrt{\bds{\pi}}_r\sqrt{\bds{\pi}}_r^\top)<1$ and thus
$$\sigma_A\triangleq\mn{A-A_\infty}_{\bds{\pi}_r} =  \sqrt{\rho(J_1-\sqrt{\bds{\pi}}_r\sqrt{\bds{\pi}}_r^\top)} 
< 1.$$
To prove~\eqref{B}, we note that~$B\mb{x}-B_\infty\mb{x} = (B-B_\infty)(\mb{x}-B_\infty\mb{x})$ and we have the following:
\begin{equation*}
\left\|B\mb{x}-B_\infty\mb{x}\right\|_{\bds{\pi}_c} \leq \mn{B-B_\infty}_{\bds{\pi}_c}\left\|\mb{x}-B_\infty\mb{x}\right\|_{\bds{\pi}_c}.
\end{equation*}
Next we show that~$\mn{B-B_\infty}_{\bds{\pi}_c}<1$. By the definition of~$\mn{\cdot}_{\bds{\pi}_c}$ in~\eqref{norm2}, we have the following:
\begin{align*}
\mn{B-B_\infty}_{\bds{\pi}_c} 
=&~  \mn{\mbox{diag}(\sqrt{\bds{\pi}_c})^{-1}(B-B_\infty)\mbox{diag}(\sqrt{\bds{\pi}_c})}_2 \triangleq \sqrt{\lambda_{\max}(H)},
\end{align*}
where~$H=\mbox{diag}(\sqrt{\bds{\pi}_c})(B-B_\infty)^\top\mbox{diag}(\bds{\pi}_c)^{-1}(B-B_\infty)\mbox{diag}(\sqrt{\bds{\pi}_c})$. Next we show that~$\rho(H)<1$. We expand the expression for~$H$ as below:
\begin{align*}
H =&~\mbox{diag}(\sqrt{\bds{\pi}_c})B^\top\mbox{diag}(\bds{\pi}_c)^{-1}B\mbox{diag}(\sqrt{\bds{\pi}_c}) 
- \mbox{diag}(\sqrt{\bds{\pi}_c})B_\infty^\top\mbox{diag}(\bds{\pi}_c)^{-1}B\mbox{diag}(\sqrt{\bds{\pi}_c}) \\
&- \mbox{diag}(\sqrt{\bds{\pi}_c})B^\top\mbox{diag}(\bds{\pi}_c)^{-1}B_\infty\mbox{diag}(\sqrt{\bds{\pi}_c})
+ \mbox{diag}(\sqrt{\bds{\pi}_c})B_\infty^\top\mbox{diag}(\bds{\pi}_c)^{-1}B_\infty\mbox{diag}(\sqrt{\bds{\pi}_c}),\\
\triangleq&~H_1 - H_2 - H_3 + H_4.
\end{align*}
With the fact that~$B_\infty = \bds{\pi}_c\mb{1}_n^\top$, one can verify that~$H_2=H_3=H_4=\sqrt{\bds{\pi}_c}\sqrt{\bds{\pi}_c}^\top$ and thus~$H = H_1 - \sqrt{\bds{\pi}}_c\sqrt{\bds{\pi}}_c^\top.$
Since~$H_1$ is primitive, by Perron-Frobenius theorem~\cite{matrix}, we have that
$$\sigma_B\triangleq\mn{B-B_\infty}_{\bds{\pi}_c} =  \sqrt{\rho(H_1-\sqrt{\bds{\pi}}_c\sqrt{\bds{\pi}}_c^\top)} < 1,$$
which completes the proof.
\end{proof}

\section*{Proof of Lemma~\ref{lem_y}}
\begin{proof}
Recall that~$B_\infty=\bds{\pi}_c\mb{1}_n^\top$. We have the following hold:
\begin{align*}
\|\mb{y}_k\|_2 \leq~& \ol{\bds{\pi}_c}^{0.5}\|\mb{y}_k - \bds{\pi}_c\mb{1}_n^\top\mb{y}_k\|_{\bds{\pi}_c} + \|\bds{\pi}_c\mb{1}_n^\top\mb{y}_k\|_{2}\\
=~& \ol{\bds{\pi}_c}^{0.5}\|\mb{y}_k - B_\infty\mb{y}_k\|_{\bds{\pi}_c} + \|\bds{\pi}_c\|_2\|\mb{1}_n^\top\mb{y}_k\|_2 \\
\leq~& \ol{\bds{\pi}_c}^{0.5}\|\mb{y}_k - B_\infty\mb{y}_k\|_{\bds{\pi}_c} + n\|\bds{\pi}_c\|_2\|\ol{\mb{y}}_k-\mb{h}(\mb{x}_k)\|_2 + n\|\bds{\pi}_c\|_2 \|\mb{h}(\mb{x}_k)-\nabla F(\widehat{\mb{x}}_k)\|_2 \nonumber\\
&+ n\|\bds{\pi}_c\|_2\|\nabla F(\widehat{\mb{x}}_k)-\nabla F(\mb{x}^*)\|_2 \\
\leq~& \ol{\bds{\pi}_c}^{0.5}\|\mb{y}_k - B_\infty\mb{y}_k\|_{\bds{\pi}_c} + n\|\bds{\pi}_c\|_2\|\ol{\mb{y}}_k-\mb{h}(\mb{x}_k)\|_2 + \frac{\sqrt{n}\|\bds{\pi}_c\|_2l}{\sqrt{\ul{\bds{\pi}_r}}}\|\mb{x}_k-\mb{1}_n\widehat{\mb{x}}_k\|_{\bds{\pi}_r}\\&+ n\|\bds{\pi}_c\|_2l\|\widehat{\mb{x}}_k-\mb{x}^*\|_2,
\end{align*}
where in the last inequality we used Lemma~\ref{lsm} and norm-equivalence.
Squaring the above, and using the basic inequality~$2ab\leq a^2+b^2$, we get
\begin{align*}
\|\mb{y}_k\|_2^2 \leq& ~4\ol{\bds{\pi}_c}\|\mb{y}_k - B_\infty\mb{y}_k\|^2_{\bds{\pi}_c} + 4n^2\|\bds{\pi}_c\|_2^2\|\ol{\mb{y}}_k-\mb{h}(\mb{x}_k)\|^2_2 \nonumber\\
&+ \frac{4n\|\bds{\pi}_c\|_2^2l^2}{\ul{\bds{\pi}_r}}\|\mb{x}_k-\mb{1}_n\widehat{\mb{x}}_k\|_{\bds{\pi}_r}^2 + 4n^2\|\bds{\pi}_c\|_2^2l^2\|\widehat{\mb{x}}_k-\mb{x}^*\|_2^2.
\end{align*}
Taking the expectation on both sides given~$\mc{F}_k$ and using~(\ref{sg}) in Lemma~\ref{SG} completes the proof.
\end{proof}

\section*{Proof of Lemma~\ref{lem_i}}
\begin{proof}
Recall that~$A_\infty = \mb{1}_n\bds{\pi}_r^\top$ and~$\widehat{\mb{x}}_{k}\triangleq\bds{\pi}_r^\top\mb{x}_k$. Using~\eqref{SABv1} and Lemma~\ref{A_B}, We have the following hold:
\begin{align*}
&\|\mb{x}_{k+1}-\mb{1}_n\widehat{\mb{x}}_{k+1}\|^2_{\bds{\pi}_r}\nonumber\\
=&~\left\|A\mb{x}_k -\alpha\mb{y}_k-A_\infty(A\mb{x}_k -\alpha\mb{y}_k)\right\|^2_{\bds{\pi}_r} \\
=&~\left\|A\mb{x}_k -A_\infty\mb{x}_k -\alpha(I_n-A_\infty)\mb{y}_k\right\|^2_{\bds{\pi}_r}\\
=&~\left\|A\mb{x}_k -A_\infty\mb{x}_k\right\|^2_{\bds{\pi}_r}
+ \alpha^2\left\|(I_n-A_\infty)\mb{y}_k\right\|^2_{\bds{\pi}_r}-2\big\langle A\mb{x}_k -A_\infty\mb{x}_k,\alpha(I_n-A_\infty)\mb{y}_k\big\rangle_{\bds{\pi}_r}\nonumber\\
\leq&~\sigma_A^2\left\|\mb{x}_k -A_\infty\mb{x}_k\right\|^2_{\bds{\pi}_r}+\alpha^2\mn{I_n-A_\infty}^2_{\bds{\pi}_r}\|\mb{y}_k\|^2_{\bds{\pi}_r}+2\alpha\sigma_A\left\|\mb{x}_k -A_\infty\mb{x}_k\right\|_{\bds{\pi}_r}\mn{I_n-A_\infty}_{\bds{\pi}_r}\|\mb{y}_k\|_{\bds{\pi}_r}\\
\leq&~\sigma_A^2\|\mb{x}_k -A_\infty\mb{x}_k\|^2_{\bds{\pi}_r} + \alpha^2\|\mb{y}_k\|^2_{\bds{\pi}_r}+\alpha\sigma_A\left(\frac{1-\sigma_A^2}{2\alpha\sigma_A}\|\mb{x}_k -A_\infty\mb{x}_k\|^2_{\bds{\pi}_r} +\frac{2\alpha\sigma_A}{1-\sigma_A^2}\|\mb{y}_k\|^2_{\bds{\pi}_r}\right)\\
\leq&~\frac{1+\sigma_A^2}{2}\|\mb{x}_k -A_\infty\mb{x}_k\|^2_{\bds{\pi}_r} + \frac{\alpha^2(1+\sigma_A^2)\ol{\bds{\pi}_r}}{1-\sigma_A^2}\|\mb{y}_k\|^2_{2}.
\end{align*}
where the second last inequality uses Young's inequality and the fact that~$\mn{I_n-A_\infty}_{\bds{\pi}_r}=1$. Taking the expectation on both sides given~$\mc{F}_k$ completes the proof.
\end{proof}

\section*{Proof of Lemma~\ref{lem_ii}}
\begin{proof}
We start by multiplying both sides of \eqref{SABv1} with~$\bds{\pi}_r^\top$ to obtain as in~\cite{AB,pushpull_Pu}:
\begin{align*}
\widehat{\mb{x}}_{k+1} - \mb{x}^*
=&~ \widehat{\mb{x}}_{k} - \mb{x}^*- \alpha \bds{\pi}_r^\top(\mb{y}_k-\bds{\pi}_c\mb{1}_n^\top\mb{y}_k+\bds{\pi}_c\mb{1}_n^\top\mb{y}_k) \\
=&~ \widehat{\mb{x}}_{k} -  \mb{x}^*- n\alpha \bds{\pi}_r^\top\bds{\pi}_c\ol{\mb{y}}_k-\alpha\bds{\pi}_r^\top(\mb{y}_k-B_\infty\mb{y}_k).
\end{align*}
Taking norms and squaring both sides leads to
\begin{align}
&\left\|\widehat{\mb{x}}_{k+1} - \mb{x}^*\right\|^2_2 \nonumber\\
=&\left\|\widehat{\mb{x}}_{k} -  \mb{x}^*- n\alpha \bds{\pi}_r^\top\bds{\pi}_c\ol{\mb{y}}_k-\alpha\bds{\pi}_r^\top(\mb{y}_k-B_\infty\mb{y}_k)\right\|^2_2\nonumber\\
=& \left\|\widehat{\mb{x}}_{k} -  \mb{x}^*- n\alpha \bds{\pi}_r^\top\bds{\pi}_c\ol{\mb{y}}_k\right\|^2_2 -2\Big\langle\widehat{\mb{x}}_{k} -  \mb{x}^*- n\alpha \bds{\pi}_r^\top\bds{\pi}_c\ol{\mb{y}}_k,\alpha\bds{\pi}_r^\top(\mb{y}_k-B_\infty\mb{y}_k)\Big\rangle+ \alpha^2\left\|\bds{\pi}_r^\top\left(\mb{y}_k-B_\infty\mb{y}_k\right)\right\|^2_2\nonumber\\
\leq& \left\|\widehat{\mb{x}}_{k} -  \mb{x}^*- n\alpha \bds{\pi}_r^\top\bds{\pi}_c\ol{\mb{y}}_k\right\|^2_2 -2\Big\langle\widehat{\mb{x}}_{k} -  \mb{x}^*- n\alpha \bds{\pi}_r^\top\bds{\pi}_c\ol{\mb{y}}_k,\alpha\bds{\pi}_r^\top(\mb{y}_k-B_\infty\mb{y}_k)\Big\rangle+ \alpha^2\|\bds{\pi}_r\|_2^2\ol{\bds{\pi}_c}\|\mb{y}_k-B_\infty\mb{y}_k\|_{\bds{\pi}_c}^2\nonumber\\
\triangleq&~r_1+r_2+\alpha^2\|\bds{\pi}_r\|_2^2\ol{\bds{\pi}_c}\|\mb{y}_k-B_\infty\mb{y}_k\|_{\bds{\pi}_c}^2. \nonumber
\end{align}
Taking the conditional expectation on bothsides given~$\mc{F}_k$, we obtain:
\begin{align}
\mathbb{E}\left[\left\|\widehat{\mb{x}}_{k+1} - \mb{x}^*\right\|^2_2\big|\mc{F}_k\right]
\leq \mathbb{E}\left[r_1|\mc{F}_k\right] + \mathbb{E}\left[r_2|\mc{F}_k\right] 
+ \alpha^2\|\bds{\pi}_r\|_2^2\ol{\bds{\pi}_c}\mathbb{E}\left[\left\|\mb{y}_k-B_\infty\mb{y}_k\right\|_{\bds{\pi}_c}^2\big|\mc{F}_k\right] \label{r}
\end{align}

\noindent \textbf{Bounding~$\mathbb{E}\left[r_1|\mc{F}_k\right]$: }We first derive an upper bound on~$r_1$. To simplify the notation, we denote~$\widetilde{\alpha}\triangleq \alpha n\bds{\pi}_r^\top\bds{\pi}_c$. If~$0<\widetilde{\alpha}<\frac{1}{l}$, we have the following:
\begin{align*}
r_1 =& \left\|\widehat{\mb{x}}_{k} -  \mb{x}^*- \widetilde{\alpha}\nabla F(\widehat{\mb{x}}_k) + \widetilde{\alpha}\nabla F(\widehat{\mb{x}}_k) - \widetilde{\alpha}\ol{\mb{y}}_k\right\|_2^2 \nonumber\\
=& \left\|\widehat{\mb{x}}_{k} -  \mb{x}^*- \widetilde{\alpha}\nabla F(\widehat{\mb{x}}_k)\right\|_2^2
+ \widetilde{\alpha}^2\left\|\nabla F(\widehat{\mb{x}}_k) - \ol{\mb{y}}_k\right\|_2^2
+ 2\widetilde{\alpha}\Big\langle\widehat{\mb{x}}_{k} -  \mb{x}^*- \widetilde{\alpha}\nabla F(\widehat{\mb{x}}_k),\nabla F(\widehat{\mb{x}}_k) - \ol{\mb{y}}_k\Big\rangle\nonumber\\
\leq&\left(1-\mu\widetilde{\alpha}\right)^2\left\|\widehat{\mb{x}}_{k}-\mb{x}^*\right\|_2^2
+ \widetilde{\alpha}^2\left\|\nabla F(\widehat{\mb{x}}_k) - \ol{\mb{y}}_k\right\|_2^2
+ 2\widetilde{\alpha}\Big\langle\widehat{\mb{x}}_{k} -  \mb{x}^*- \widetilde{\alpha}\nabla F(\widehat{\mb{x}}_k),\nabla F(\widehat{\mb{x}}_k) - \ol{\mb{y}}_k\Big\rangle,
\end{align*}
where in the last inequality above we used Lemma~\ref{cvx_stan}. Then we take the conditional expectation given~$\mc{F}_k$ on bothsides to obtain:
\begin{align}
&\mathbb{E}\left[r_1|\mc{F}_k\right]\nonumber\\
=&\left(1-\mu\widetilde{\alpha}\right)^2\left\|\widehat{\mb{x}}_{k}-\mb{x}^*\right\|_2^2
+ \widetilde{\alpha}^2\mathbb{E}\left[\left\|\nabla F(\widehat{\mb{x}}_k) - \ol{\mb{y}}_k\right\|_2^2\big|\mc{F}_k\right]
+ 2\widetilde{\alpha}\Big\langle\widehat{\mb{x}}_{k} -  \mb{x}^*- \widetilde{\alpha}\nabla F(\widehat{\mb{x}}_k),\nabla F(\widehat{\mb{x}}_k) - \mb{h}\left(\mb{x}_k\right)\Big\rangle \nonumber\\
\leq&\left(1-\mu\widetilde{\alpha}\right)^2\left\|\widehat{\mb{x}}_{k}-\mb{x}^*\right\|_2^2
+ \widetilde{\alpha}^2\mathbb{E}\left[\left\|\nabla F(\widehat{\mb{x}}_k) - \ol{\mb{y}}_k\right\|_2^2\big|\mc{F}_k\right]
+ 2\widetilde{\alpha}\left(1-\mu\widetilde{\alpha}\right)\left\|\widehat{\mb{x}}_{k} -  \mb{x}^*\right\|_2\left\|\nabla F(\widehat{\mb{x}}_k) - \mb{h}\left(\mb{x}_k\right)\right\|_2 \nonumber\\
\leq& \left(1-\mu\widetilde{\alpha}\right)^2\left\|\widehat{\mb{x}}_{k}-\mb{x}^*\right\|_2^2
+ \widetilde{\alpha}^2\mathbb{E}\left[\left\|\nabla F(\widehat{\mb{x}}_k) - \ol{\mb{y}}_k\right\|_2^2\big|\mc{F}_k\right]
\nonumber\\
&+ \widetilde{\alpha}\left(\mu\left(1-\mu\widetilde{\alpha}\right)^2\left\|\widehat{\mb{x}}_{k}-\mb{x}^*\right\|_2^2+
\tfrac{1}{\mu}\left\|\nabla F(\widehat{\mb{x}}_k) - \mb{h}\left(\mb{x}_k\right)\right\|_2^2\right)\nonumber\\
\leq& \left(1-\mu\widetilde{\alpha}\right)\left\|\widehat{\mb{x}}_{k}-\mb{x}^*\right\|_2^2 + \widetilde{\alpha}^2\mathbb{E}\left[\left\|\nabla F(\widehat{\mb{x}}_k) - \ol{\mb{y}}_k\right\|_2^2\big|\mc{F}_k\right] + \frac{\widetilde{\alpha}l^2}{\mu n\ul{\bds{\pi}_r}}\left\|\mb{x}_k-\mb{1}_n\widehat{\mb{x}}_k\right\|_{\bds{\pi}_r}^2,
\label{r11}
\end{align}
where in the second last inequality we used Lemma~\ref{cvx_stan} and Young's inequality and in the last inequality we used Lemma~\ref{lsm} and the fact that~$\left(1-\mu\widetilde{\alpha}\right)^2\left(1+\mu\widetilde{\alpha}\right)<\left(1-\mu\widetilde{\alpha}\right)$. In order to finish bounding~$r_1$, we next bound~$\mathbb{E}\left[\left\|\nabla F(\widehat{\mb{x}}_k) - \ol{\mb{y}}_k\right\|_2^2\big|\mc{F}_k\right]$:
\begin{align}
&\mathbb{E}\left[\left\|\nabla F(\widehat{\mb{x}}_k) - \ol{\mb{y}}_k\right\|_2^2\big|\mc{F}_k\right] \nonumber\\ 
=&~\mathbb{E}\left[\left\|\nabla F(\widehat{\mb{x}}_k) -\mb{h}(\mb{x}_k) + \mb{h}(\mb{x}_k) - \ol{\mb{y}}_k\right\|_2^2\big|\mc{F}_k\right]\nonumber\\
=&\left\|\nabla F(\widehat{\mb{x}}_k) -\mb{h}(\mb{x}_k)\right\|_2^2 + \mathbb{E}\left[\left\|\mb{h}(\mb{x}_k) - \ol{\mb{y}}_k\right\|_2^2\big|\mc{F}_k\right]\nonumber
\end{align}
Using Lemma~\ref{SG} (3) and Lemma~\ref{lsm} leads to:
\begin{align}
\mathbb{E}\left[\left\|\nabla F(\widehat{\mb{x}}_k) - \ol{\mb{y}}_k\right\|_2^2\big|\mc{F}_k\right]
\leq \frac{l^2}{n\ul{\bds{\pi}_r}}\left\|\mb{x}_k-\mb{1}_n\widehat{\mb{x}}_k\right\|_{\bds{\pi}_r}^2 + \frac{\sigma^2}{n} \label{r12}
\end{align}
Plugging in~\eqref{r12} to~\eqref{r11}, we obtain an upper bound on~$\mathbb{E}\left[r_1|\mc{F}_k\right]$ as follows:
\begin{align}
\mathbb{E}\left[r_1|\mc{F}_k\right]\leq&
\left(1-\mu\widetilde{\alpha}\right)\left\|\widehat{\mb{x}}_{k}-\mb{x}^*\right\|_2^2 
+ \frac{\widetilde{\alpha}l^2}{n\ul{\bds{\pi}_r}}\left(\frac{1}{\mu}+\widetilde{\alpha}\right)
\left\|\mb{x}_k-\mb{1}_n\widehat{\mb{x}}_k\right\|_{\bds{\pi}_r}^2 + \frac{2\widetilde{\alpha}^2\sigma^2}{n}\nonumber\\
\leq&
\left(1-\mu\widetilde{\alpha}\right)\left\|\widehat{\mb{x}}_{k}-\mb{x}^*\right\|_2^2 
+ \frac{2\widetilde{\alpha}l^2}{\mu n\ul{\bds{\pi}_r}}
\left\|\mb{x}_k-\mb{1}_n\widehat{\mb{x}}_k\right\|_{\bds{\pi}_r}^2 + \frac{\widetilde{\alpha}^2\sigma^2}{n}\label{r1}
\end{align}
\noindent \textbf{Bounding~$\mathbb{E}\left[r_2|\mc{F}_k\right]$: } Recall that~$r_2=-2\alpha\big\langle\widehat{\mb{x}}_{k} -  \mb{x}^*- 
\widetilde{\alpha}\ol{\mb{y}}_k,\bds{\pi}_r^\top(\mb{y}_k-B_\infty\mb{y}_k)\big\rangle$.
\begin{align*}
r_2 \leq&~2\alpha\left\|\widehat{\mb{x}}_{k} -  \mb{x}^*- 
\widetilde{\alpha}\ol{\mb{y}}_k\right\|_2\left\|\bds{\pi}_r^\top(\mb{y}_k-B_\infty\mb{y}_k)\right\|_2\\
\leq&~\alpha\left(\frac{\mu n\bds{\pi}_r^\top\bds{\pi}_c}{2}r_1+\frac{2\left\|\bds{\pi}_r\right\|_2^2\ol{\bds{\pi}_c}}{\mu n\bds{\pi}_r^\top\bds{\pi}_c}\left\|\mb{y}_k-B_\infty\mb{y}_k\right\|^2_{\bds{\pi}_c}\right)
\end{align*}
Taking the conditional expectation on bothsides given~$\mc{F}_k$ to get:
\begin{align}
\mathbb{E}\left[r_2|\mc{F}_k\right]
\leq&~\frac{\mu\widetilde{\alpha}}{2}\mathbb{E}\left[r_1|\mc{F}_k\right]
+ \frac{2\alpha\left\|\bds{\pi}_r\right\|_2^2\ol{\bds{\pi}_c}}{\mu n\bds{\pi}_r^\top\bds{\pi}_c}\mathbb{E}\left[\left\|\mb{y}_k-B_\infty\mb{y}_k\right\|^2_{\bds{\pi}_c}\big|\mc{F}_k\right] \label{r2}
\end{align}
\noindent \textbf{Bounding~$\mathbb{E}\left[r_1|\mc{F}_k\right]+\mathbb{E}\left[r_2|\mc{F}_k\right]$: } Putting~\eqref{r1} and~\eqref{r2} together, we have:
\begin{align} 
&\mathbb{E}\left[r_1|\mc{F}_k\right]+\mathbb{E}\left[r_2|\mc{F}_k\right]\nonumber\\
\leq&~\left(1+\frac{\mu\widetilde{\alpha}}{2}\right)\frac{2\widetilde{\alpha}l^2}{\mu n\ul{\bds{\pi}_r}}\left\|\mb{x}_k-\mb{1}_n\widehat{\mb{x}}_k\right\|_{\bds{\pi}_r}^2
+ \left(1+\frac{\mu\widetilde{\alpha}}{2}\right)\left(1-\mu\widetilde{\alpha}\right)\left\|\widehat{\mb{x}}_{k}-\mb{x}^*\right\|_2^2
+
\left(1+\frac{\mu\widetilde{\alpha}}{2}\right)\frac{\widetilde{\alpha}^2\sigma^2}{n}\nonumber\\
&+\frac{2\alpha\left\|\bds{\pi}_r\right\|_2^2\ol{\bds{\pi}_c}}{\mu n\bds{\pi}_r^\top\bds{\pi}_c}\left\|\mb{y}_k-B_\infty\mb{y}_k\right\|^2_{\bds{\pi}_c}\nonumber\\
\leq&~\frac{3\widetilde{\alpha}l^2}{\mu n\ul{\bds{\pi}_r}}\left\|\mb{x}_k-\mb{1}_n\widehat{\mb{x}}_k\right\|_{\bds{\pi}_r}^2
+\left(1-\frac{\mu\widetilde{\alpha}}{2}\right)\left\|\widehat{\mb{x}}_k-\mb{x}^*\right\|_{2}^2
+\frac{2\alpha\left\|\bds{\pi}_r\right\|_2^2\ol{\bds{\pi}_c}}{\mu n\bds{\pi}_r^\top\bds{\pi}_c}\left\|\mb{y}_k-B_\infty\mb{y}_k\right\|^2_{\bds{\pi}_c} + \frac{3\widetilde{\alpha}^2\sigma^2}{ 2n},\label{r1r2}
\end{align}
where in the last inequality we use the fact that~$1+\frac{\mu\widetilde{\alpha}}{2}\leq\frac{3}{2}$ and~$(1+\frac{\mu\widetilde{\alpha}}{2})\left(1-\mu\widetilde{\alpha}\right)\leq1-\frac{\mu\widetilde{\alpha}}{2}$.
Plugging~\eqref{r1r2} in~\eqref{r} and replacing~$\widetilde{\alpha}$ by~$ \alpha n\bds{\pi}_r^\top\bds{\pi}_c$ finishes the proof.
\end{proof}

\section*{Proof of Lemma~\ref{lem_iii}}
\begin{proof}
To simplify notation, we define:
\begin{align*}
\begin{array}{ll}
\nabla_k \triangleq \nabla\mb{f}(\mb{x}_k), \qquad&
\nabla_k^i\triangleq\nabla f_i(\mb{x}_k^i),
\\
\mb{g}_k\triangleq\mb{g}(\mb{x}_{k},\bds{\xi}_k),&\mb{g}_k^i\triangleq\mb{g}_i(\mb{x}_{k}^i,\xi_{k}^i).\\
\end{array}
\end{align*}
Starting with~\eqref{SABv2} and using Lemma~\ref{A_B}, we obtain
\begin{align}
&\left\|\mb{y}_{k+1}-B_\infty\mb{y}_{k+1}\right\|_{\bds{\pi}_c}^2\nonumber\\ =&~\|B\mb{y}_{k}-B_\infty\mb{y}_{k}\|_{\bds{\pi}_c}^2\nonumber+\left\|\left(I_n-B_\infty\right)\left(\mb{g}_{k+1}-\mb{g}_k\right)\right\|_{\bds{\pi}_c}^2
+2\Big\langle B\mb{y}_{k}-B_\infty\mb{y}_{k}, (I_n-B_\infty)(\mb{g}_{k+1}-\mb{g}_k)\Big\rangle_{\bds{\pi}_c}\nonumber\\
\leq&~ \sigma_B^2 \left\|\mb{y}_{k}-B_\infty\mb{y}_{k}\right\|_{\bds{\pi}_c}^2 + \left\|\mb{g}_{k+1}-\mb{g}_k\right\|_{\bds{\pi}_c}^2 +2\Big\langle B\mb{y}_{k}-B_\infty\mb{y}_{k}, \mb{g}_{k+1}-\mb{g}_k\Big\rangle_{\bds{\pi}_c}, \label{t0} 
\end{align}
where the last inequality uses~$\mn{I_n-B_\infty}_{\bds{\pi}_c} = 1$ and that
\begin{align*}
&\Big\langle B\mb{y}_{k}-B_\infty\mb{y}_{k}, B_\infty\left(\mb{g}_{k+1}-\mb{g}_k\right)\Big\rangle_{\bds{\pi}_c}=
\Big\langle B\mb{y}_{k}-B_\infty\mb{y}_{k}, \mb{1}_n\mb{1}_n^\top\left(\mb{g}_{k+1}-\mb{g}_k\right)\Big\rangle=0.
\end{align*}
We take the conditional expectation given~$\mc{F}_k$ on both sides of~\eqref{t0} to get:
\begin{align}
\mathbb{E}&\left[\left\|\mb{y}_{k+1}-B_\infty\mb{y}_{k+1}\right\|_{\bds{\pi}_c}^2\big|\mc{F}_k\right] \nonumber\\
\leq
&~\sigma_B^2\mathbb{E}\left[\left\|\mb{y}_{k}-B_\infty\mb{y}_{k}\right\|_{\bds{\pi}_c}^2\big|\mc{F}_k\right]
+
\mathbb{E}\left[\left\|\mb{g}_{k+1}-\mb{g}_k\right\|_{\bds{\pi}_c}^2\big|\mc{F}_k\right]
+2\mathbb{E}\left[\mathbb{E}\left[\big\langle B\mb{y}_{k}-B_\infty\mb{y}_{k}, \mb{g}_{k+1}-\mb{g}_k\big\rangle_{\bds{\pi}_c}\Big|\mc{F}_{k+1}\right]\Big|\mc{F}_k\right]\nonumber\\
=&~\sigma_B^2\mathbb{E}\left[\left\|\mb{y}_{k}-B_\infty\mb{y}_{k}\right\|_{\bds{\pi}_c}^2\big|\mc{F}_k\right]
+
\mathbb{E}\left[\left\|\mb{g}_{k+1}-\mb{g}_k\right\|_{\bds{\pi}_c}^2\big|\mc{F}_k\right]
+2\mathbb{E}\left[\Big\langle B\mb{y}_{k}-B_\infty\mb{y}_{k}, \nabla_k-\mb{g}_k\Big\rangle_{\bds{\pi}_c}\Big|\mc{F}_k\right] 
\nonumber\\
&+2\mathbb{E}\left[\Big\langle B\mb{y}_{k}-B_\infty\mb{y}_{k}, \nabla_{k+1}-\nabla_{k}\Big\rangle_{\bds{\pi}_c}\Big|\mc{F}_k\right],\nonumber\\
\triangleq&~\sigma_B^2\mathbb{E}\left[\left\|\mb{y}_{k}-B_\infty\mb{y}_{k}\right\|_{\bds{\pi}_c}^2\big|\mc{F}_k\right]+s_1+2s_2+s_3. \label{s}
\end{align}
We now bound the last three terms in the following. We start with~$s_1$.

\noindent \textbf{Bounding~$s_1$: }
\begin{align}
s_1 =&~\mathbb{E}\left[\left\|\mb{g}_{k+1}-\mb{g}_k-(\nabla_{k+1}-\nabla_{k})+\nabla_{k+1}-\nabla_{k}\right\|_{\bds{\pi}_c}^2\big|\mc{F}_k\right]\nonumber\\
=&~\mathbb{E}\left[\left\|\nabla_{k+1}-\nabla_{k}\right\|_{\bds{\pi}_c}^2\big|\mc{F}_k\right]
\nonumber\\&+\mathbb{E}\left[\left\|\mb{g}_{k+1}-\mb{g}_k-(\nabla_{k+1}-\nabla_{k})\right\|_{\bds{\pi}_c}^2\big|\mc{F}_k\right]+2\mathbb{E}\left[ \Big\langle\nabla_{k+1}-\nabla_{k},\mb{g}_{k+1}-\mb{g}_k-\left(\nabla_{k+1}-\nabla_{k}\right)\Big\rangle_{\bds{\pi}_c}\Big|\mc{F}_k\right]\nonumber\\
\leq&~\mathbb{E}\left[\left\|\nabla_{k+1}-\nabla_{k}\right\|_{\bds{\pi}_c}^2|\mc{F}_k\right]
+ 2\mathbb{E}\left[\Big\langle\nabla_{k+1},\nabla_{k}-\mb{g}_k\Big\rangle_{\bds{\pi}_c}\Big|\mc{F}_k\right] 
+\frac{2n\sigma^2}{\ul{\bds{\pi}_c}}
\label{s1_0}.
\end{align}
We now bound the first term in the above inequality. Using~\eqref{SABv1}, Lemma~\ref{A_B}, we have
\begin{align}
&\left\|\nabla_{k+1}-\nabla_{k}\right\|_{\bds{\pi}_c}^2
\nonumber\\
\leq&~ \tfrac{l^2}{\ul{\bds{\pi}_c}}\left\|\mb{x}_{k+1}-\mb{x}_{k}\right\|^2_{2}\nonumber\\
=&~ \tfrac{l^2}{\ul{\bds{\pi}_c}}\left\|A\mb{x}_{k}-\alpha\mb{y}_k-\mb{x}_{k}\right\|^2_{2}\nonumber\\
=&~\tfrac{l^2}{\ul{\bds{\pi}_c}}\left\|\left(A-I_n\right)\left(\mb{x}_{k}-A_\infty\mb{x}_{k}\right)-\alpha\mb{y}_k\right\|^2_{2}\nonumber\\
\leq&~\tfrac{8l^2}{\ul{\bds{\pi}_c}\:\ul{\bds{\pi}_r}}\left\|\mb{x}_{k}-A_\infty\mb{x}_{k}\right\|_{\bds{\pi}_r}^2+\tfrac{2\alpha^2l^2}{\ul{\bds{\pi}_c}}\left\|\mb{y}_k\right\|_{2}^2, \label{s1_1}
\end{align}
where in the last inequality we used the basic inequality~$\left\|\mb{x}+\mb{y}\right\|_2^2\leq2\left\|\mb{x}\right\|_2^2+2\left\|\mb{y}\right\|_2^2, \forall\mb{x},\mb{y}\in\mathbb{R}^p$ and that~$\mn{A-I_n}_{\bds{\pi_r}}\leq2$.   
In order to bound the second term in~\eqref{s1_0}, we first note, from~\eqref{SAB1}, that
\begin{align*}
&\nabla_{k+1}^i=\nabla f_i\left(\sum_{j=1}^{n}a_{ij}\mb{x}_k^j \right.\nonumber\left.-\alpha\left(\sum_{j=1}^{n}b_{ij}\mb{y}_{k-1}^j + \mb{g}_{k}^i-\mb{g}_{k-1}^i\right)\right),
\end{align*} 
and we also define
$$\widetilde{\nabla}_{k+1}^i\triangleq\nabla f_i\left(\sum_{j=1}^{n}a_{ij}\mb{x}_k^j -\alpha\left(\sum_{j=1}^{n}b_{ij}\mb{y}_{k-1}^j + \nabla_{k}^i-\mb{g}_{k-1}^i\right)\right).$$
Therefore,~$\big\|\nabla_{k+1}^i-\widetilde{\nabla}_{k+1}^i\big\|_2\leq\alpha l\left\|\nabla_k^i-\mb{g}_{k}^i\right\|_2$.
We then proceed to bound the second term in~\eqref{s1_0} as follows:
\begin{align}
&\mathbb{E}\left[\Big\langle\nabla_{k+1},\nabla_{k}-\mb{g}_k\Big\rangle_{\bds{\pi}_c}|\mc{F}_k\right]\nonumber\\
\leq&~\frac{1}{\ul{\bds{\pi}_c}}\sum_{i=1}^{n}\mathbb{E}\left[\Big\langle\nabla_{k+1}^i-\widetilde{\nabla}_k^i+\widetilde{\nabla}_k^i,\nabla_{k}^i-\mb{g}_{k}^i\Big\rangle\Big|\mc{F}_k\right]\nonumber\\
\leq&~\frac{1}{\ul{\bds{\pi}_c}}\sum_{i=1}^{n}\mathbb{E}\left[\big\|\nabla_{k+1}^i-\widetilde{\nabla}_k^i\big\|_2 \left\|\nabla _{k}^i-\mb{g}_{k}^i\right\|_2\big|\mc{F}_k\right] \nonumber\\
\leq&~\frac{\alpha l}{\ul{\bds{\pi}_c}}\sum_{i=1}^{n}\mathbb{E}\left[\left\|\nabla _{k}^i-\mb{g}_{k}^i\right\|_2^2\big|\mc{F}_k\right]\nonumber\\ 
\leq&~\frac{\alpha n l\sigma^2}{\ul{\bds{\pi}_c}}. \label{s1_2}
\end{align}
Plugging~\eqref{s1_1} and~\eqref{s1_2} to~\eqref{s1_0}, we obtain an upper bound on~$s_1$:
\begin{align}
s_1\leq \frac{8l^2}{\ul{\bds{\pi}_c}\:\ul{\bds{\pi}_r}}\left\|\mb{x}_{k}-A_\infty\mb{x}_{k}\right\|_{\bds{\pi}_r}^2+\frac{2\alpha^2l^2}{\ul{\bds{\pi}_c}}\mathbb{E}\left[\left\|\mb{y}_k\right\|_{2}^2\big|\mc{F}_k\right]+\frac{2\alpha ln \sigma^2}{\ul{\bds{\pi}_c}}+\frac{2n\sigma^2}{\ul{\bds{\pi}_c}}. \label{s1}
\end{align}

\noindent\textbf{Bounding~$s_2$:} We first split~$s_2$ into the sum of two terms as follows:
\begin{align}
s_2=&~\mathbb{E}\left[\Big\langle B\mb{y}_{k}, \nabla_{k}-\mb{g}_k\Big\rangle_{\bds{\pi}_c}\Big|\mc{F}_k\right]
-\mathbb{E}\left[\Big\langle B_\infty\mb{y}_{k}, \nabla_{k}-\mb{g}_k\Big\rangle_{\bds{\pi}_c}|\mc{F}_k\right] \label{s2_0}.
\end{align}
For the first term in~\eqref{s2_0}, using~\eqref{SABv2}, we have
\begin{align}
\mathbb{E}&\left[\Big\langle B\mb{y}_{k},  \nabla_{k}-\mb{g}_k\Big\rangle_{\bds{\pi}_c}\Big|\mc{F}_k\right]
\nonumber\\
=&~
\mathbb{E}\left[\big\langle B^2\mb{y}_{k-1}+B(\mb{g}_k-\mb{g}_{k-1}), \nabla_{k}-\mb{g}_k\big\rangle_{\bds{\pi}_c}\big|\mc{F}_k\right]\nonumber\\
=&~\mathbb{E}\left[\big\langle B\mb{g}_k,  \nabla_{k}-\mb{g}_k\big\rangle_{\bds{\pi}_c}\big|\mc{F}_k\right]\nonumber\\
=&~\sum_{i=1}^{n}\frac{1}{[\bds{\pi}_c]_{i}}\mathbb{E}\left[\Big\langle\sum_{j=1}^{n}b_{ij}\mb{g}_k^j,\nabla_{k}^i-\mb{g}_{k}^i\Big\rangle
\Big|\mc{F}_k\right]\nonumber\\
=&~\sum_{i=1}^{n}\frac{b_{ii}}{[\bds{\pi}_c]_{i}}\mathbb{E}\left[\Big\langle\mb{g}_{k}^i,\nabla_{k}^i-\mb{g}_{k}^i\Big\rangle\big|\mc{F}_k\right]\leq 0.\nonumber
\end{align}
For the second term in~\eqref{s2_0}, we have 
\begin{align*}
-\mathbb{E}&\left[\Big\langle B_\infty\mb{y}_{k}, \nabla_{k}-\mb{g}_k\Big\rangle_{\bds{\pi}_c}\Big|\mc{F}_k\right]
\nonumber\\
=&~-\mathbb{E}\left[\Big\langle \mb{1}_n\mb{1}_n^\top\mb{g}_k, \nabla_{k}-\mb{g}_k\Big\rangle\Big|\mc{F}_k\right]\nonumber\\
=&~\sum_{i=1}^{n}\mathbb{E}\left[\Big\langle \sum_{j=1}^{n}-\mb{g}_k^j, \nabla_{k}^i-\mb{g}_k^i\Big\rangle\Big|\mc{F}_k\right]\nonumber\\
=&~\sum_{i=1}^{n}\mathbb{E}\left[\Big\langle \nabla_{k}^i-\mb{g}_k^i, \nabla_{k}^i-\mb{g}_k^i\Big\rangle\Big|\mc{F}_k\right]
\leq n\sigma^2
\end{align*}
Hence, we have an upper bound on~$s_2$ as follows: 
\begin{equation}
s_2\leq n\sigma^2. \label{s2}
\end{equation}
%

\noindent \textbf{Bounding~$s_3$: } Using the upper bound on~$\left\|\nabla_{k+1}-\nabla_{k}\right\|_{\bds{\pi}_c}^2$ in~\eqref{s1_1}, we proceed towards an upper bound on~$s_3$.
\begin{align}
&2\big\langle B\mb{y}_{k}-B_\infty\mb{y}_{k}, \nabla_{k+1}-\nabla_{k}\big\rangle_{\bds{\pi}_c}
\nonumber\\
\leq&~2\left\|B\mb{y}_{k}-B_\infty\mb{y}_{k}\right\|_{\bds{\pi}_c}\left\|\nabla_{k+1}-\nabla_{k}\right\|_{\bds{\pi}_c}\nonumber\\
\leq&~\frac{1-\sigma_B^2}{2\sigma_B^2}\left\|B\mb{y}_{k}-B_\infty\mb{y}_{k}\right\|_{\bds{\pi}_c}^2+\frac{2\sigma_B^2}{1-\sigma_B^2}\left\|\nabla_{k+1}-\nabla_{k}\right\|_{\bds{\pi}_c}^2\nonumber\\
\leq&~\frac{1-\sigma_B^2}{2}\left\|\mb{y}_{k}-B_\infty\mb{y}_{k}\right\|_{\bds{\pi}_c}^2
+\frac{16\sigma_B^2l^2}{\ul{\bds{\pi}_c}\:\ul{\bds{\pi}_r}(1-\sigma_B^2)}\left\|\mb{x}_{k}-A_\infty\mb{x}_{k}\right\|_{\bds{\pi}_r}^2+\frac{4\sigma_B^2l^2\alpha^2}{\ul{\bds{\pi}_c}(1-\sigma_B^2)}\left\|\mb{y}_k\right\|_{2}^2\nonumber.
\end{align}
Taking the conditional expectation given~$\mc{F}_k$ on bothsides of the inequality above, we obtain an upper bound on~$s_3$ as follows:
\begin{align}
s_3 \leq \frac{16\sigma_B^2l^2}{\ul{\bds{\pi}_c}\:\ul{\bds{\pi}_r}(1-\sigma_B^2)}\left\|\mb{x}_{k}-A_\infty\mb{x}_{k}\right\|_{\bds{\pi}_r}^2 + \frac{1-\sigma_B^2}{2}\mathbb{E}\left[\left\|\mb{y}_{k}-B_\infty\mb{y}_{k}\right\|_{\bds{\pi}_c}^2\big|\mc{F}_k\right]
+\frac{4\sigma_B^2l^2\alpha^2}{\ul{\bds{\pi}_c}(1-\sigma_B^2)}\mathbb{E}\left[\left\|\mb{y}_k\right\|_{2}^2|\mc{F}_k\right] \label{s3}
\end{align}
\end{proof}
Plugging the bounds on~$s_1,s_2,s_3$ in~\eqref{s1},~\eqref{s2} and~\eqref{s3} into~\eqref{s} completes the proof.
\end{document}